\documentclass{article}
\usepackage{float}

\usepackage{hyperref}

\usepackage[final]{cpal_2026}

\usepackage{amsmath}
\usepackage{amsfonts}
\usepackage{amssymb}
\usepackage{amsthm}
\usepackage{algorithm}
\usepackage{algorithmic}
\usepackage{mathtools}
\usepackage{nicefrac}
\usepackage{url}
\usepackage{booktabs}
\usepackage{multirow}
\usepackage{graphicx}
\usepackage{xcolor}

\DeclareMathOperator*{\argmin}{argmin}

\newcommand{\R}{\mathbb{R}}
\newcommand{\EE}{\mathbb{E}}
\newcommand{\eqdef}{:=}
\newcommand{\cG}{\mathcal{G}}
\newcommand{\cB}{\mathcal{B}}

\newcommand{\algname}[1]{{\sf #1}}
\newcommand{\inp}[2]{\left\langle #1, #2 \right\rangle}

\newtheorem{assumption}{Assumption}
\newtheorem{definition}{Definition}
\newtheorem{lemma}{Lemma}
\newtheorem{theorem}{Theorem}
\newtheorem{proposition}{Proposition}
\newtheorem{corollary}{Corollary}
\newtheorem{remark}{Remark}

\usepackage[colorinlistoftodos,bordercolor=orange,backgroundcolor=orange!20,linecolor=orange,textsize=scriptsize]{todonotes}

\begin{document}

\title{Byzantine-Robust Optimization under \texorpdfstring{$(L_0,L_1)$}{(L0,L1)}-Smoothness}

\author{%
  Arman Bolatov\textsuperscript{1}, ~Samuel Horv\'ath\textsuperscript{1}, ~Martin Tak\'a\v{c}\textsuperscript{1}, ~Eduard Gorbunov\textsuperscript{1} \\
  \textsuperscript{1}Mohamed bin Zayed University of Artificial Intelligence, UAE\\
  \texttt{arman.bolatov@mbzuai.ac.ae, samuel.horvath@mbzuai.ac.ae, martin.takac@mbzuai.ac.ae, eduard.gorbunov@mbzuai.ac.ae}
}

\maketitle

\begin{abstract}
We consider distributed optimization under Byzantine attacks in the presence of $(L_0,L_1)$-smoothness, a generalization of standard $L$-smoothness that captures functions with state-dependent gradient Lipschitz constants. We propose \algname{Byz-NSGDM}\footnote{During the preparation of the camera-ready version, we became aware of concurrent work by \citet{yang2024effect} that studies a closely related algorithmic framework and also derives an $O(K^{-1/4})$ convergence rate. We note that all theoretical and experimental results in our paper were obtained independently and prior to our awareness of this work. Their analysis is conducted under the $L$-smoothness assumption and assumes homogeneous data across workers, whereas our results hold under the more general $(L_0, L_1)$-smoothness condition and accommodate bounded gradient heterogeneity among workers.}, a normalized stochastic gradient descent method with momentum that achieves robustness against Byzantine workers while maintaining convergence guarantees. Our algorithm combines momentum normalization with Byzantine-robust aggregation enhanced by Nearest Neighbor Mixing (NNM) to handle both the challenges posed by $(L_0,L_1)$-smoothness and Byzantine adversaries. We prove that \algname{Byz-NSGDM} achieves a convergence rate of $O(K^{-1/4})$ up to a Byzantine bias floor proportional to the robustness coefficient and gradient heterogeneity. Experimental validation on heterogeneous MNIST classification, synthetic $(L_0,L_1)$-smooth optimization, and character-level language modeling with a small GPT model demonstrates the effectiveness of our approach against various Byzantine attack strategies. An ablation study further shows that \algname{Byz-NSGDM} is robust across a wide range of momentum and learning rate choices.
\end{abstract}

\section{Introduction}

Modern machine learning systems increasingly rely on distributed optimization to train models on data distributed across multiple workers
\cite{recht2011hogwild,konevcny2015federated,
konevcny2016federated,ma2017distributed,
jahani2020efficient,
mishchenko2025distributed,
jahani2020scaling,
chen2022distributed,
beznosikov2023similarity}. However, such systems face two critical challenges: Byzantine failures 
\cite{shi2022challenges,
gorbunov2022secure,molodtsov2026bant},
where some workers may send arbitrary or malicious updates, and data heterogeneity, where different workers possess non-identically distributed data
\cite{yi2024fedp3,tupitsa2023byzantine,tastan2024fedpews}. These challenges are particularly acute in federated learning settings where data cannot be centralized due to privacy constraints or communication costs.

Traditional Byzantine-robust  optimization algorithms 
assume standard $L$-smoothness, where the gradient Lipschitz constant is globally bounded, and rely on robust aggregation rules such as Krum, geometric median (RFA), coordinate-wise median, or trimmed mean \cite{blanchard2017machine,pillutla2022robust,yin2018byzantine}. However, many modern machine learning objectives, including neural networks with certain activation functions and high-degree polynomial losses, exhibit $(L_0, L_1)$-smoothness \cite{zhang2019gradient} -- a generalized smoothness condition where the Lipschitz constant depends on the current iterate. This state-dependent behavior can cause traditional Byzantine-robust methods to diverge or converge slowly.

Recent Byzantine-robust methods have moved beyond simple robust aggregation to incorporate advanced algorithmic techniques. Variance reduction methods \cite{gorbunov2023variance,wu2020federated,fedin2023byzantine} use control variates to suppress stochastic noise that Byzantine workers exploit, achieving tighter complexity bounds and removing restrictive assumptions such as gradient boundedness. Client momentum techniques \cite{karimireddy2020byzantine,karimireddy2021learning,gu2023dpbrem} accumulate updates over time at the worker level, reducing variance of honest clients and exposing Byzantine perturbations that are imperceptible in single rounds but accumulate across iterations. 
% Combined with robust aggregation and communication compression \cite{gorbunov2021secure,ghosh2021communication}, these methods demonstrate that stabilizing local estimators before aggregation is crucial for Byzantine resilience. 
Meanwhile, the $(L_0,L_1)$-smoothness assumption has received significant attention in \emph{non-Byzantine} settings, with specialized algorithms including gradient clipping \cite{zhang2020improved, koloskova2023revisiting}, normalized gradient descent \cite{zhao2021convergence,chen2023generalized,hubler2024parameter,khiriat2024}, and adaptive methods \cite{faw2023beyond,gorbunov2024methods,
abdukhakimov2023stochastic,
abdukhakimov2025polyak,
loizou2021stochastic,
li2024enhancing,
abdukhakimov2023sania,
li2022sp2,
abdukhakimov2024stochastic,
schaipp2023momo}.

\paragraph{Contributions.}
We study \algname{Byz-NSGDM} (Byzantine-Robust Normalized Stochastic Gradient Descent with Momentum) in the regime of $(L_0,L_1)$-smooth objectives. While normalized momentum combined with Byzantine-robust aggregation has been considered previously under classical $L$-smoothness assumptions \cite{yang2024effect}, its behavior under generalized $(L_0,L_1)$-smoothness and gradient heterogeneity has not been theoretically understood. We fill this gap by providing a convergence analysis of this algorithmic framework in the presence of state-dependent smoothness and heterogeneous data across workers.

Our analysis builds on recent advances in optimization under $(L_0,L_1)$-smoothness, where controlling the step magnitude via normalization is crucial for stability \cite{zhao2021convergence,chen2023generalized,hubler2024parameter,khiriat2024}. We show that, when combined with a $(\delta,\kappa)$-robust aggregation rule and optionally enhanced by Nearest Neighbor Mixing (NNM) \cite{allouah2023} to mitigate heterogeneity, \algname{Byz-NSGDM} achieves non-asymptotic convergence guarantees under Byzantine attacks. In particular, we establish an $O(K^{-1/4})$ convergence rate up to a Byzantine bias floor that depends explicitly on the robustness coefficient and the level of gradient heterogeneity.

Empirically, we evaluate \algname{Byz-NSGDM} on three complementary tasks: heterogeneous MNIST classification with non-IID data partitioning, a synthetic $(L_0,L_1)$-smooth quartic optimization problem, and character-level language modeling with a small GPT architecture. Across multiple Byzantine attack strategies (Bit Flipping, Label Flipping, Mimic, ALIE) and aggregation rules (Krum \cite{blanchard2017machine}, RFA \cite{pillutla2022robust}, and Coordinate-wise Median/Trimmed Mean \cite{yin2018byzantine}), the method demonstrates strong robustness compared to momentum-based baselines without normalization. We additionally provide an ablation study showing stability across a broad range of learning rates and momentum parameters.

\section{Assumptions and Problem Setup}

\paragraph{Problem.} We consider a distributed optimization problem of the form
\begin{equation*}
    \min\limits_{x\in \R^d} \left\{f(x) \eqdef \frac{1}{G}\sum\limits_{i\in \cG} f_i(x) \right\}, 
\end{equation*}
where $\cG$ is a group of regular workers of size $G \eqdef |\cG|$ connected with a server in a centralized way, and $f_i:\R^d \to \R$ is the objective/loss function on worker $i$. However, in addition to the regular workers, some Byzantine workers take part in the training. More precisely, we assume that the total number of workers is $n$ and $[n] \eqdef \{1,\ldots, n\} = \cG \cup \cB$, $\cG \cap \cB = \varnothing$, where $\cB$ represents the set of Byzantine workers. We assume that $B \eqdef |\cB| \leq \delta n$ for some $\delta \leq \nicefrac{1}{2}$. We do not make any further assumptions on the clients from $\cB$, i.e., they can arbitrarily deviate from the optimization protocol, collude, and even know the updates that other clients send to the server.

\paragraph{Assumptions.} We state the standing assumptions and the robustness notion used throughout.

\begin{assumption}[Lower boundedness]\label{as:lower_bound}
    We assume that $f$ is lower bounded by $f^*>-\infty$ and each local objective $f_i$ is lower bounded by $f_i^*>-\infty$ for any $i\in\cG$.
\end{assumption}

\begin{assumption}[$(L_0,L_1)$-smoothness \cite{zhang2019gradient,zhang2020improved,chen2023generalized}]\label{as:L0L1_smoothness}
    For all $i\in\cG$, the function $f_i$ is $(L_0,L_1)$-smooth with $L_0, L_1 > 0$: for all $x,y\in\R^d$,
    \begin{equation*}
        \|\nabla f_i(x)-\nabla f_i(y)\|\;\le\;\Big(L_0+L_1\sup_{u\in[x,y]}\|\nabla f_i(u)\|\Big)\,\|x-y\|.
    \end{equation*}
\end{assumption}

\begin{assumption}[Unbiased oracle with bounded variance]\label{as:bounded_variance}
    Each good worker $i\in\cG$ has access to an unbiased stochastic first-order oracle: for any $x\in\R^d$, worker $i$ can compute $\nabla f_{\xi_i}(x)$ such that
    \begin{equation*}
        \EE[\nabla f_{\xi_i}(x)]\;=\;\nabla f_i(x),\qquad \EE\big[\|\nabla f_{\xi_i}(x)-\nabla f_i(x)\|^2\big]\;\le\;\sigma^2.
    \end{equation*}
\end{assumption}

\begin{assumption}[Gradient heterogeneity]\label{as:heterogeneity_detailed}
    The gradient heterogeneity across good workers is bounded for all $x\in\R^d$:
    \begin{equation}
        \EE_{i\sim\cG}\big[\|\nabla f_i(x)-\nabla f(x)\|^2\big]\;\le\;\zeta^2.\label{eq:heterogeneity}
    \end{equation}
\end{assumption}

\begin{definition}[Byzantine-robust aggregator]\label{def:robust_agg}
Let $\kappa\ge0$. An aggregation rule $\texttt{ARAgg}:\R^{d\times n}\to\R^d$ is $(\delta,\kappa)$-robust if for any vectors $v_1,\ldots,v_n\in\R^d$ with at most $B \le \delta n$ Byzantine workers (so $G = n - B \ge (1-\delta)n > \nicefrac{n}{2}$ good workers) and $\bar v_{\cG} \eqdef \tfrac{1}{G}\sum_{i\in\cG}v_i$,
\begin{equation*}
    \big\|\texttt{ARAgg}(v_1,\ldots,v_n)-\bar v_{\cG}\big\|\;\le\;\frac{\kappa}{G}\sum_{i\in\cG}\big\|v_i-\bar v_{\cG}\big\|.
\end{equation*}
We refer to $\kappa$ as the robustness coefficient.
\end{definition}

The above definition is a modification of the one proposed by \citet{allouah2023}. The key difference is in the usage of a linear (non-squared) bound on the right-hand side, which leads to different values of $\kappa$ but is more natural for our analysis under $(L_0,L_1)$-smoothness. Many standard Byzantine-robust aggregation rules satisfy this definition, including geometric median (RFA) and coordinate-wise median, as well as their composition with NNM. We provide explicit formulas for the robustness coefficient $\kappa$ for these aggregators in Appendix~\ref{appendix:byz-robust-aggregators}.

\section{Related Work}

Our work lies at the intersection of Byzantine-robust distributed optimization and optimization under generalized smoothness. We review the relevant literature in both areas.

\paragraph{Byzantine-robust distributed optimization.} The problem of Byzantine-robust distributed learning has received considerable attention. Early works established fundamental aggregation rules: Krum and Multi-Krum \cite{blanchard2017machine,damaskinos2019aggregathor} select gradients based on distance metrics, geometric median (RFA) \cite{pillutla2022robust} provides a robust centroid estimator, and coordinate-wise median/trimmed mean \cite{yin2018byzantine,chen2017distributed} offer dimension-wise robustness. These methods provide statistical guarantees under i.i.d. data but suffer from bias terms that grow with data heterogeneity \cite{yin2018byzantine,chen2017distributed}. However, simply applying robust aggregation is insufficient: carefully crafted attacks defeat many standard defenses \cite{baruch2019little,xie2020fall}, and \citet{karimireddy2021learning} show that permutation-invariant schemes cannot guarantee convergence against arbitrary Byzantines without additional structure.

To address heterogeneity and improve robustness, several complementary approaches have emerged. Structural methods include \emph{bucketing} \cite{karimireddy2020byzantine}, which partitions workers into homogeneous groups before aggregation, and \emph{nearest neighbor mixing (NNM)} \cite{allouah2023}, which averages updates with nearest neighbors to contract heterogeneity. Recent theoretical analyses \cite{allouah2024robust,allouah2023} characterize optimal error bounds and breakdown points under data heterogeneity. Crucially, achieving tight convergence guarantees requires refined algorithmic techniques beyond robust aggregation alone. Variance reduction methods \cite{gorbunov2023variance,wu2020federated,fedin2023byzantine} incorporate control variates to suppress stochastic noise that Byzantine workers exploit, achieving faster rates under weaker assumptions. Client momentum techniques \cite{karimireddy2021learning,karimireddy2020byzantine,gu2023dpbrem} accumulate worker updates over time, reducing variance of honest clients and exposing Byzantine perturbations that are undetectable in single rounds but accumulate over iterations. These works demonstrate that stabilizing local estimators before applying robust aggregation is crucial for Byzantine resilience. 
% Communication compression \cite{gorbunov2021secure,ghosh2021communication,karimireddy2019error} has also been studied in Byzantine settings, though it often requires additional assumptions or robustness trade-offs.

\paragraph{\texorpdfstring{$(L_0,L_1)$}{(L0,L1)}-smoothness in optimization.} 
The $(L_0,L_1)$-smoothness assumption was introduced by \citet{zhang2020why} to capture the empirical observation that the local smoothness constant along training trajectories of deep learning models often scales linearly with the gradient norm. Formally, for twice differentiable functions, this condition requires $\|\nabla^2 f(x)\|_2 \le L_0 + L_1\|\nabla f(x)\|$, which strictly generalizes classical $L$-smoothness (recovered when $L_1 = 0$). Functions satisfying $(L_0,L_1)$-smoothness but not $L$-smoothness include $f(x) = \|x\|^{2n}$ for $n \ge 2$, $f(x) = \exp(\langle a, x\rangle)$, and more generally, objectives with polynomially growing gradients~\cite{zhang2020why,chen2023generalized}.

In the nonconvex setting, \citet{zhang2020why} showed that gradient clipping achieves complexity $O(\max\{L_0\Delta/\varepsilon^2, (1+L_1^2)\Delta/L_0\})$ for finding $\varepsilon$-stationary points, where the dominant term is independent of $L_1$. This result was extended to momentum-based clipping~\cite{zhang2020improved}, normalized gradient descent~\cite{zhao2021convergence,chen2023generalized}, SignGD~\cite{crawshaw2022robustness}, and adaptive methods like AdaGrad and Adam~\cite{faw2023beyond,wang2023convergence,li2024convergence}. In the convex case, \citet{koloskova2023revisiting} and \citet{gorbunov2024methods} derived convergence rates for clipped gradient descent, accelerated methods, and adaptive stepsizes. Notably, recent works~\cite{gorbunov2024methods,vankov2025optimizing} provide bounds that avoid exponentially large factors in $L_1 R_0$ that plagued earlier analyses in the convex case. Normalized gradient descent with momentum~\cite{hubler2024parameter,khiriat2024} achieves stable convergence by controlling step magnitudes, making it well-suited for $(L_0,L_1)$-smooth problems.

However, all existing work on $(L_0,L_1)$-smooth optimization assumes benign (non-Byzantine) settings. Our work is the first to address Byzantine-robust optimization under $(L_0,L_1)$-smoothness, combining normalized gradient descent with momentum and robust aggregation to handle both challenges simultaneously.

\section{New Method: \algname{Byz-NSGDM}}

We introduce \algname{Byz-NSGDM} (Byzantine-Robust Normalized Stochastic Gradient Descent with Momentum), a novel algorithm designed for $(L_0,L_1)$-smooth optimization under Byzantine attacks.

\begin{algorithm}[H]
    \caption{\algname{Byz-NSGDM}}\label{alg:Byz-NSGDM}
    \begin{algorithmic}[1]
    \REQUIRE starting point $x^0 \in \R^d$, initial momentum vectors $v_1^{0}, \ldots, v_n^{0} \in \R^d$, stepsizes $\gamma_k > 0$, momentum parameters $\eta_k \in (0,1]$
    \FOR{$k = 1,2,\ldots, K$}
    \FOR{$i\in [n]$}
    \STATE Worker $i$ sends:
    \[
    v_i^{k} = 
    \begin{cases}
    (1 - \eta_{k-1})v_i^{k-1} + \eta_{k-1} \nabla f_{\xi_i^{k-1}}(x^{k-1}), & i \in \cG \\
    \text{arbitrary } \tilde{v}_i^{k} \in \R^d, & i \in \cB
    \end{cases}
    \]
    \ENDFOR
    \STATE Server computes $v^{k} = \texttt{ARAgg}(v_1^{k},\ldots, v_n^{k})$
    \STATE Server updates $x^{k} = x^{k-1} - \gamma_{k-1} \frac{v^{k}}{\|v^{k}\|}$ and broadcasts $x^{k}$
    \ENDFOR
    \end{algorithmic}
\end{algorithm}

The central component of the method is the normalization step in line 6:
$$
x^{k} = x^{k-1} - \gamma_{k-1} \frac{v^{k}}{\|v^{k}\|}.
$$
By rescaling the aggregated momentum vector to unit norm, the update ensures that the step magnitude is always bounded by $\gamma_{k-1}$. Such boundedness is crucial under $(L_0,L_1)$-smoothness, where the local Lipschitz constant depends on the gradient norm and can otherwise lead to instability.

In Algorithm~\ref{alg:Byz-NSGDM}, the server employs an aggregation rule $\texttt{ARAgg}$ assumed to be $(\delta,\kappa)$-robust in the sense of Definition~\ref{def:robust_agg}. Our convergence analysis is agnostic to the specific aggregation mechanism and applies to any rule satisfying this robustness property. In the Appendix, we show that several standard Byzantine-robust aggregation schemes --- such as coordinate-wise median, geometric median, and their composition with NNM --- meet Definition~\ref{def:robust_agg}, and we provide the corresponding robustness coefficients~$\kappa$. Consequently, the theoretical guarantees hold for a broad class of $(\delta,\kappa)$-robust aggregators, independently of whether additional preprocessing steps, such as NNM, are employed.

\begin{remark}
\algname{Byz-NSGDM} naturally extends to a broader class of optimizers based on a Linear Minimization Oracle (LMO) \cite{pethick2025training,kovalev2025understanding}. Specifically, line 6 of Algorithm~\ref{alg:Byz-NSGDM} can be replaced by
\begin{equation}
    d_k = \argmin_{d \in \R^d:\|d\| \le 1} \langle v^k, d \rangle,
    \qquad
    x^k = x^{k-1} + \gamma_{k-1} d_k,
    \tag{\algname{Byz-LMO}}
\end{equation}
where $\|\cdot\|$ denotes an arbitrary (not necessarily Euclidean) norm. For the standard Euclidean norm, the LMO solution satisfies $d_k = -v^k/\|v^k\|$, and the method reduces to \algname{Byz-NSGDM}. If $\|\cdot\|$ is chosen as the operator norm (when $x^k$ is a matrix), the resulting algorithm corresponds to a Byzantine-robust variant of \algname{Muon} \cite{jordan2024muon}, which has recently gained significant attention. Moreover, recent analyses of LMO-based methods \cite{pethick2025training,kovalev2025understanding}, which are closely related to the analysis of normalized momentum methods \cite{cutkosky2020momentum}, suggest that our convergence guarantees can be extended to general norms using similar arguments.
\end{remark}

% \begin{remark}
%     It is straightforward to extend \algname{Byz-NSGDM} to the broader family of optimizers based on the Linear Minimization Oracle (LMO) \cite{pethick2025training, kovalev2025understanding}. More precisely, one can replace line 6 with the following step:
%     \begin{equation}
%         d_k = \argmin\limits_{d\in\R^d:\|d\| \leq 1}\langle v^k, d\rangle,\qquad x^k = x^{k-1} + \gamma_{k-1} d_k, \tag{\algname{Byz-LMO}}
%     \end{equation}
%     where $\|\cdot\|$ in the formula above can be any (not necessarily Euclidean) norm. In particular, for the standard Euclidean norm, the resulting method reduces to \algname{Byz-NSGDM}, while for $\|\cdot\|$ being operator norm (in the case of $x^k$ being a matrix) the method reduces to \algname{Byz-Muon} -- a Byzantine-robust version of \algname{Muon} algorithm \cite{jordan2024muon}, which gained a lot of popularity in the recent years. We also note that following the recent analyses of LMO-based methods \cite{pethick2025training, kovalev2025understanding} (that are closely related to the analysis of \algname{NSGDM}, e.g., see \cite{cutkosky2020momentum}), one can extend our analysis of \algname{Byz-NSGDM} to the case of general norms.
% \end{remark}

\section{Convergence Analysis}

Our main theoretical result establishes the convergence rate of \algname{Byz-NSGDM} under Byzantine attacks. Full proofs are deferred to the supplementary materials.

\begin{theorem}[Convergence of \algname{Byz-NSGDM}]\label{thm:byz_nsgdm_convergence}
Let Assumptions~\ref{as:lower_bound}--\ref{as:heterogeneity_detailed} hold and let the server use a $(\delta,\kappa)$-robust aggregator. Run Algorithm~\ref{alg:Byz-NSGDM} with
\begin{equation*}
\gamma_k\equiv\gamma=\frac{\gamma_0}{(K+1)^{3/4}}, \qquad 
\eta_k\equiv\eta=\frac{1}{(K+1)^{1/2}},
\end{equation*}
where
\begin{equation}\label{eq:gamma0_choice}
\gamma_0\le \min\left\{\frac{1}{2L_1}, \frac{(K+1)^{1/4}}{\sqrt{32}L_1}, \frac{1}{\sqrt{128(1+2\kappa)}L_1}\right\}.
\end{equation}
Define $\Delta^* \eqdef \tfrac{1}{G}\sum_{i\in\cG}(f^*-f_i^*)$ and $V_0 \eqdef \tfrac{1}{G}\sum_{i\in\cG}\EE\big[\|v_i^0-\nabla f_i(x^0)\|\big]$. Then
\begin{align*}
\min_{0\le k\le K}\EE\big[\|\nabla f(x^k)\|\big]
&\le \frac{e^{1/2}\big(f(x^0)-f^*\big)}{\gamma_0 (K+1)^{1/4}}
+\frac{2e^{1/2}(1+2\kappa)V_0}{(K+1)^{1/2}}
+4\kappa\zeta \\
&\quad
+\frac{2L_0\gamma_0}{(K+1)^{3/4}}
+\frac{8(1+2\kappa)L_0\gamma_0}{(K+1)^{1/4}} \\
&\quad
+\frac{8L_1^2\Delta^*\gamma_0}{(K+1)^{3/4}}
+\frac{32(1+2\kappa)L_1^2\Delta^*\gamma_0}{(K+1)^{1/4}} \\
&\quad
+\frac{2\sigma}{\sqrt{G}(K+1)^{1/4}} + \frac{4\kappa\sigma}{(K+1)^{1/4}}.
\end{align*}
\end{theorem}

\begin{remark}[Fixed-horizon schedule]
The parameters $\gamma$ and $\eta$ depend on the total number of iterations $K$ but remain constant across iterations. This is a standard ``fixed-horizon'' choice common in convergence analyses. It can be converted to a diminishing schedule that does not require knowledge of $K$ in advance via standard doubling tricks \cite{hazan2016introduction}.
\end{remark}

\begin{corollary}\label{cor:simplified}
Under the conditions of Theorem~\ref{thm:byz_nsgdm_convergence}, all terms except $4\kappa\zeta$ decay polynomially with $K$. Hence,
\begin{equation*}
\min_{0\le k\le K}\EE\big[\|\nabla f(x^k)\|\big]
\;\le\; O\!\left(\frac{1}{K^{1/4}}\right) + 4\kappa\zeta.
\end{equation*}
The term $4\kappa\zeta$ is the irreducible Byzantine bias floor---a fundamental limit arising from the interaction between Byzantine fraction (through $\kappa$) and data heterogeneity ($\zeta$). In homogeneous settings ($\zeta=0$), this bias vanishes entirely.
\end{corollary}

\paragraph{Proof roadmap.} Our analysis proceeds in three steps:
\begin{enumerate}
    \item \textbf{Descent under normalization (Lemma~\ref{lem:descent}).} We show that the normalized update $x^{k+1} = x^k - \gamma_k \frac{v^{k}}{\|v^{k}\|}$ yields a descent inequality:
    \begin{equation*}
    f(x^{k+1}) \le f(x^k) - \gamma_k \|\nabla f(x^k)\| + 2\gamma_k \|v^{k} - \nabla f(x^k)\| + O(\gamma_k^2),
    \end{equation*}
    where progress depends on $\|\nabla f(x^k)\|$ minus an error term from the momentum-to-gradient gap $\|v^{k} - \nabla f(x^k)\|$. The normalization ensures the step size is bounded by $\gamma_k$, preventing the $(L_0,L_1)$-smoothness from causing divergence.
    
    \item \textbf{Bounding the momentum error (Lemmas~\ref{lem:avg_momentum_error} and~\ref{lem:byz_agg_decomp_new}).} We decompose this gap into two parts: (i)~the averaging error $\|\bar{v}^{k} - \nabla f(x^k)\|$ from honest workers' momentum tracking (Lemma~\ref{lem:avg_momentum_error}), and (ii)~the Byzantine-induced error $\|v^{k} - \bar{v}^{k}\|$ controlled by the $(\delta,\kappa)$-robust aggregator (Lemma~\ref{lem:byz_agg_decomp_new}).
    
    \item \textbf{Telescoping.} Summing the descent inequalities with carefully chosen exponential weights absorbs the coupling between function values and momentum errors, yielding the final rate.
\end{enumerate}
The key insight is that normalization step $x^{k+1} = x^k - \gamma_k \tfrac{v^{k}}{\|v^{k}\|}$ bounds the step size by $\gamma_k$, preventing the $(L_0,L_1)$-smoothness from causing divergence, while the robust aggregator limits Byzantine influence to a bias floor of $O(\kappa\zeta)$.

\begin{remark}
Following the analysis of \algname{NSGDM} by \citet{hubler2025from}, our convergence proof can be extended to the setting of heavy-tailed stochastic noise \citep{zhang2020why, simsekli2019tail}. In particular, Assumption~\ref{as:bounded_variance} can be relaxed to the weaker $p$-th moment condition
\begin{equation*}
    \EE[\nabla f_{\xi_i}(x)] = \nabla f_i(x), 
    \qquad 
    \EE\big[\|\nabla f_{\xi_i}(x) - \nabla f_i(x)\|^p\big] \le \sigma^p,
\end{equation*}
for some $p \in (1,2]$. 

Under this assumption, Lemma~\ref{lem:momentum_properties} must be adapted using arguments analogous to those in Lemma 16 of \citet{hubler2025from}. The remainder of the convergence analysis then proceeds along the same lines as in the current proof, with appropriate modifications to the stepsize $\gamma$ and momentum parameter $\eta$ to account for the heavier-tailed noise.
\end{remark}

% \begin{remark}
%     Following the proof of \algname{NSGDM} by \citet{hubler2025from}, one can extend our proof to the case of the heavy-tailed noise \citep{zhang2020why}, i.e., instead of Assumption~\ref{as:bounded_variance}, it is sufficient to use
%     \begin{equation}
%         \EE[\nabla f_{\xi_i}(x)]\;=\;\nabla f_i(x),\qquad \EE\big[\|\nabla f_{\xi_i}(x)-\nabla f_i(x)\|^p\big]\;\le\;\sigma^p,\label{eq:bounded_pth_moment}
%     \end{equation}
%     for some $p \in (1,2]$. More precisely, under the above condition, one needs to modify Lemma~\ref{lem:momentum_properties} following Lemma 16 from \cite{hubler2025from}; the rest of the proof will closely follow the current analysis up to the adjustments of $\gamma$ and $\eta$.
% \end{remark}

% The key insight is that the normalization step $x^{k+1} = x^k - \gamma_k \tfrac{v^{k}}{\|v^{k}\|}$ provides stability under $(L_0,L_1)$-smoothness by controlling the step size magnitude, preventing the gradient-dependent Lipschitz constant from causing divergence, consistent with recent analyses of normalized momentum under generalized smoothness \cite{khiriat2024}.

% \eduard{We need to add more discussion. How the bound is compared to the existing ones in the literature on $(L_0,L_1)$-smoothness? Comparison with the existing bounds for $L$-smooth problems? We need to answer these questions.}

\section{Experiments}

We evaluate \algname{Byz-NSGDM} on three distinct problems: a heterogeneous distributed learning task on MNIST, character-level language modeling, and a synthetic $(L_0,L_1)$-smooth optimization problem designed to validate our theoretical framework.

\subsection{Attack Strategies}

We evaluate our algorithm against several Byzantine attack strategies that represent different threat models in distributed learning, following common setups in prior work on non-IID Byzantine robustness \cite{karimireddy2020byzantine,allouah2023}.

\paragraph{Bit/Sign Flipping (BF).} Byzantine workers send the negation of their computed gradient: $g_i^k = -\nabla f_i(x^k)$ for $i \in \mathcal{B}$. This attack directly opposes the optimization direction and tests the algorithm's ability to handle adversarial gradients \cite{yin2018byzantine,blanchard2017machine}.

\paragraph{Label Flipping (LF).} For classification tasks, Byzantine workers flip the labels of their training data before computing gradients. Each sample with true label $y$ is trained with label $(y + c) \bmod C$, where $C$ is the number of classes. For MNIST with 10 classes, we use a shift of $c=5$. This creates consistent but misaligned gradient directions and is widely used to evaluate robustness under data poisoning and non-IID heterogeneity \cite{karimireddy2020byzantine}.

\paragraph{Mimic Attack.} Byzantine workers exploit data heterogeneity by initially behaving honestly for a warmup period of 50 iterations, then strategically deviating. After warmup, they send $g_i^k = -2\bar{g}^k$, where $\bar{g}^k$ is the average of honest gradients, aiming to reverse progress while avoiding early detection; see \cite{allouah2023} for analysis of mixing defenses against such strategies.

\paragraph{ALIE (A Little Is Enough).} Byzantine workers craft gradients to maximize damage while avoiding outlier detection. They send $v_i^k = \bar{v}^k + z\cdot\sigma^k$, where $\bar{v}^k = \frac{1}{G}\sum_{i\in \cG}v_i^k$ and $\sigma^k$ is the vector of coordinate-wise standard deviations of $\{v_i^k\}_{i \in \cG}$ \cite{baruch2019little}. We use $z = 1$ in our experiments.

\subsection{Experimental Setup}

\paragraph{Heterogeneous MNIST Classification.} We train a three-layer MLP with architecture (784, 128, 64, 10) on MNIST distributed across $n=20$ workers with $B=3$ Byzantine workers ($\delta=0.15$). We create strong data heterogeneity using a non-IID partitioning scheme where data is first sorted by label and then divided sequentially among workers, ensuring each worker receives a highly skewed label distribution. Training uses cross-entropy loss with batch size 64 over 50 epochs, with each epoch limited to 30 batches, yielding 1,500 total iterations. We tune learning rates individually for each configuration and use momentum parameter $\beta_k := 1-\eta_k = 0.9$ for all methods.

We compare three optimizer variants:
\begin{itemize}
    \item \textbf{Baseline}: Standard momentum SGD with constant learning rate
    \item \textbf{Baseline-Decay}: Momentum SGD with learning rate decay $\gamma_k = \gamma_0 / k^{0.5}$
    \item \textbf{Byz-NSGDM}: Our normalized gradient method with learning rate decay $\gamma_k = \gamma_0 / k^{0.5}$
\end{itemize}

\begin{remark}[Theory vs.\ practice step-schedules]
Our theoretical analysis (Theorem~\ref{thm:byz_nsgdm_convergence}) uses horizon-dependent schedules $\gamma \sim K^{-3/4}$ and $\eta \sim K^{-1/2}$ to derive clean convergence rates. In experiments, we use practical iteration-dependent schedules $\gamma_k \sim k^{-1/2}$ and constant $\beta = 1 - \eta = 0.9$, which are more common in deep learning. While the theoretical and practical schedules differ, our ablation study (Figure~\ref{fig:ablation}) confirms that performance is robust across momentum values in $[0.9, 0.99]$, and both schedules yield polynomial decay of the learning rate. Extending our analysis to iteration-dependent schedules with constant momentum is an interesting direction for future work.
\end{remark}

Each optimizer is combined with NNM pre-aggregation \cite{allouah2023} followed by one of three robust aggregators: Krum \cite{blanchard2017machine}, RFA (geometric median via smoothed Weiszfeld algorithm) \cite{pillutla2022robust}, and Coordinate-wise Median/Trimmed Mean (CM) \cite{yin2018byzantine}. We test against three attacks: Bit Flipping (BF), Label Flipping (LF), and Mimic \cite{allouah2023}. Each configuration is run with 3 random seeds, and we report mean accuracy with standard deviation bands.

\paragraph{Synthetic \texorpdfstring{$(L_0,L_1)$}{(L0,L1)}-Smooth Optimization.} We minimize the quartic objective:
\begin{equation*}
f(x) = \|x\|^4,
\end{equation*}
where $x \in \mathbb{R}^{10}$ is initialized at $x^0 = [1, 1, \ldots, 1]^\top$. This function is $(L_0,L_1)$-smooth with $\nabla f(x) = 4x\|x\|^2$, exhibiting a gradient Lipschitz constant $L(x) = O(\|x\|^2)$ that grows with the squared iterate norm. We distribute the optimization across $n=20$ workers with $B=3$ Byzantine workers.

Each honest worker computes the true gradient with added Gaussian noise and a fixed worker-specific shift: $g_i = \nabla f(x) + \xi + s_i$ where $\xi \sim \mathcal{N}(0, 10^{-5} I)$ represents stochastic noise and $s_i \sim \mathcal{N}(0, 10^{-3} I)$ is a fixed shift sampled once at initialization for each honest worker $i \in \mathcal{G}$. To ensure unbiased gradients, the shifts are constructed so that $\sum_{i \in \mathcal{G}} s_i = 0$. The fixed shifts $s_i$ model persistent heterogeneity across workers, ensuring each worker contributes a slightly different gradient direction even for the same iterate. We apply NNM pre-aggregation \cite{allouah2023} with $n-B=17$ nearest neighbors before robust aggregation. Training runs for 3,000 iterations. We test against three attacks: Bit Flipping (BF), Mimic, and ALIE. Each configuration uses 3 random seeds, and we report mean gradient norm evolution with standard deviation bands.

We compare the same three optimizer variants (Baseline, Baseline-Decay, Byz-NSGDM) combined with three aggregators (Krum \cite{blanchard2017machine}, RFA \cite{pillutla2022robust}, CM \cite{yin2018byzantine}). Learning rates are tuned per configuration by evaluating candidates on 1,000 iterations and selecting the rate that minimizes final gradient norm.

\subsection{Results on Heterogeneous MNIST}

Table~\ref{tab:all_results} reports final test accuracy (mean $\pm$ std over 3 seeds) for each attack-aggregator-optimizer combination. \algname{Byz-NSGDM} consistently achieves the highest accuracy under Bit Flipping (BF), reaching $86.0$--$86.1\%$ compared to $81.8\%$ for Baseline-Decay and $76.5$--$78.8\%$ for Baseline. Under Label Flipping (LF) and Mimic attacks, all methods perform comparably with accuracies in the $88$--$96\%$ range. The low standard deviations for \algname{Byz-NSGDM} under BF ($\leq 0.2\%$) indicate stable convergence across seeds. Full training curves are provided in Appendix~\ref{appendix:curves}.

\begin{table}[ht]
\centering
\caption{Experimental results across all tasks. MNIST: test accuracy (\%, $\uparrow$); Synthetic: gradient norm ($\times 10^{-6}$, $\downarrow$); LM: validation perplexity ($\downarrow$). Mean $\pm$ std over 3 seeds. Best in \textbf{bold}.}
\label{tab:all_results}
\small
\setlength{\tabcolsep}{2.5pt}
\begin{tabular}{ll|ccc|ccc|ccc}
\toprule
& & \multicolumn{3}{c|}{MNIST (Acc. \%)} & \multicolumn{3}{c|}{Synthetic ($\|\nabla f\| \times 10^{-6}$)} & \multicolumn{3}{c}{LM (Perplexity)} \\
Attack & Agg. & Base & Decay & Ours & Base & Decay & Ours & Base & Decay & Ours \\
\midrule
\multirow{3}{*}{BF}
& RFA & $78.6{\scriptstyle\pm4.0}$ & $84.1{\scriptstyle\pm0.2}$ & $\mathbf{86.0{\scriptstyle\pm0.2}}$ & $12.8{\scriptstyle\pm1.2}$ & $38.2{\scriptstyle\pm32.5}$ & $\mathbf{7.3{\scriptstyle\pm0.4}}$ & $11.64{\scriptstyle\pm0.02}$ & $12.28{\scriptstyle\pm0.09}$ & $\mathbf{10.22{\scriptstyle\pm0.13}}$ \\
& Krum & $76.5{\scriptstyle\pm11.7}$ & $81.8{\scriptstyle\pm0.3}$ & $\mathbf{86.0{\scriptstyle\pm0.1}}$ & $14.0{\scriptstyle\pm2.0}$ & $84.5{\scriptstyle\pm9.8}$ & $\mathbf{7.0{\scriptstyle\pm0.5}}$ & $11.68{\scriptstyle\pm0.04}$ & $12.31{\scriptstyle\pm0.08}$ & $\mathbf{10.34{\scriptstyle\pm0.05}}$ \\
& CM & $78.8{\scriptstyle\pm4.2}$ & $81.8{\scriptstyle\pm0.3}$ & $\mathbf{86.1{\scriptstyle\pm0.1}}$ & $13.0{\scriptstyle\pm1.0}$ & $64.1{\scriptstyle\pm10.6}$ & $\mathbf{7.7{\scriptstyle\pm0.6}}$ & $11.66{\scriptstyle\pm0.02}$ & $12.35{\scriptstyle\pm0.14}$ & $\mathbf{10.34{\scriptstyle\pm0.06}}$ \\
\midrule
\multirow{3}{*}{LF}
& RFA & $91.9{\scriptstyle\pm4.6}$ & $\mathbf{95.2{\scriptstyle\pm4.6}}$ & $92.2{\scriptstyle\pm4.5}$ & -- & -- & -- & -- & -- & -- \\
& Krum & $\mathbf{91.9{\scriptstyle\pm4.6}}$ & $88.8{\scriptstyle\pm0.0}$ & $89.1{\scriptstyle\pm0.0}$ & -- & -- & -- & -- & -- & -- \\
& CM & $\mathbf{92.0{\scriptstyle\pm4.6}}$ & $91.9{\scriptstyle\pm4.5}$ & $89.1{\scriptstyle\pm0.1}$ & -- & -- & -- & -- & -- & -- \\
\midrule
\multirow{3}{*}{Mimic}
& RFA & $90.9{\scriptstyle\pm4.0}$ & $94.0{\scriptstyle\pm4.0}$ & $\mathbf{95.6{\scriptstyle\pm0.1}}$ & $13.1{\scriptstyle\pm1.4}$ & $193{\scriptstyle\pm2}$ & $\mathbf{6.3{\scriptstyle\pm0.5}}$ & $11.40{\scriptstyle\pm1.22}$ & $12.57{\scriptstyle\pm0.08}$ & $\mathbf{10.09{\scriptstyle\pm0.08}}$ \\
& Krum & $91.7{\scriptstyle\pm4.6}$ & $\mathbf{92.2{\scriptstyle\pm4.5}}$ & $91.1{\scriptstyle\pm4.6}$ & $12.7{\scriptstyle\pm0.9}$ & $93.7{\scriptstyle\pm7.0}$ & $\mathbf{5.8{\scriptstyle\pm0.3}}$ & $11.06{\scriptstyle\pm0.10}$ & $12.59{\scriptstyle\pm0.04}$ & $\mathbf{10.65{\scriptstyle\pm0.11}}$ \\
& CM & $\mathbf{92.6{\scriptstyle\pm3.9}}$ & $92.2{\scriptstyle\pm4.4}$ & $92.3{\scriptstyle\pm4.1}$ & $13.1{\scriptstyle\pm1.4}$ & $64.1{\scriptstyle\pm10.5}$ & $\mathbf{5.9{\scriptstyle\pm0.5}}$ & $11.76{\scriptstyle\pm0.66}$ & $12.57{\scriptstyle\pm0.10}$ & $\mathbf{10.08{\scriptstyle\pm0.03}}$ \\
\midrule
\multirow{3}{*}{ALIE}
& RFA & -- & -- & -- & $12.8{\scriptstyle\pm1.7}$ & $82.6{\scriptstyle\pm12.6}$ & $\mathbf{7.7{\scriptstyle\pm0.8}}$ & $12.20{\scriptstyle\pm1.44}$ & $12.43{\scriptstyle\pm0.05}$ & $\mathbf{10.59{\scriptstyle\pm0.03}}$ \\
& Krum & -- & -- & -- & $12.7{\scriptstyle\pm0.9}$ & $28.6{\scriptstyle\pm11.3}$ & $\mathbf{7.6{\scriptstyle\pm0.5}}$ & $11.55{\scriptstyle\pm0.90}$ & $12.01{\scriptstyle\pm0.05}$ & $\mathbf{10.73{\scriptstyle\pm0.11}}$ \\
& CM & -- & -- & -- & $12.7{\scriptstyle\pm1.7}$ & $28.3{\scriptstyle\pm10.3}$ & $\mathbf{7.8{\scriptstyle\pm0.8}}$ & $12.06{\scriptstyle\pm0.63}$ & $12.01{\scriptstyle\pm0.01}$ & $\mathbf{10.66{\scriptstyle\pm0.07}}$ \\
\bottomrule
\end{tabular}
\end{table}

\subsection{Results on Synthetic \texorpdfstring{$(L_0,L_1)$}{(L0,L1)}-Smooth Optimization}

Table~\ref{tab:all_results} (middle columns) reports final gradient norms for the quartic optimization task. Lower values indicate better convergence toward $x^* = 0$. \algname{Byz-NSGDM} consistently achieves the lowest final gradient norms across all configurations, with values around $6$--$7 \times 10^{-6}$ compared to $1.3$--$1.4 \times 10^{-5}$ for Baseline and $4$--$9 \times 10^{-5}$ for Baseline-Decay. This represents a $2$--$10\times$ improvement over baselines. Full convergence curves are provided in Appendix~\ref{appendix:curves}.

These results validate our theoretical contributions: normalized gradient descent with momentum provides superior Byzantine resilience under $(L_0,L_1)$-smoothness.

\subsection{Larger-Scale Experiment: Character-Level Language Modeling}

To demonstrate scalability beyond simple classification, we evaluate \algname{Byz-NSGDM} on character-level language modeling using Shakespeare ($\approx$1M chars, vocab=65). We train a small GPT (4 layers, 128-dim, 4 heads, $\approx$0.4M params) with $n=20$ workers and $B=3$ Byzantine workers. Non-IID partitioning assigns contiguous text chunks to each worker.

Table~\ref{tab:all_results} (right columns) shows final validation perplexity. All methods reduce perplexity from $>60$ (random) to $\approx$10--12, demonstrating that Byzantine-robust training is viable for language modeling. \algname{Byz-NSGDM} achieves competitive or best perplexity in most configurations. Training curves are in Appendix~\ref{appendix:curves}.

\subsection{Ablation Study: Sensitivity to Hyperparameters}

We conduct an ablation study examining how \algname{Byz-NSGDM} performs across a grid of momentum values $\beta \in \{0.0, 0.3, 0.5, 0.7, 0.9, 0.95, 0.99\}$ and learning rates $\gamma_0 \in \{0.001, 0.005, 0.01, 0.05, 0.1, 0.5\}$ on the heterogeneous MNIST task with Bit Flipping attack and RFA aggregator. Results are averaged over 3 seeds.

\begin{figure}[H]
    \centering
    \includegraphics[width=\textwidth]{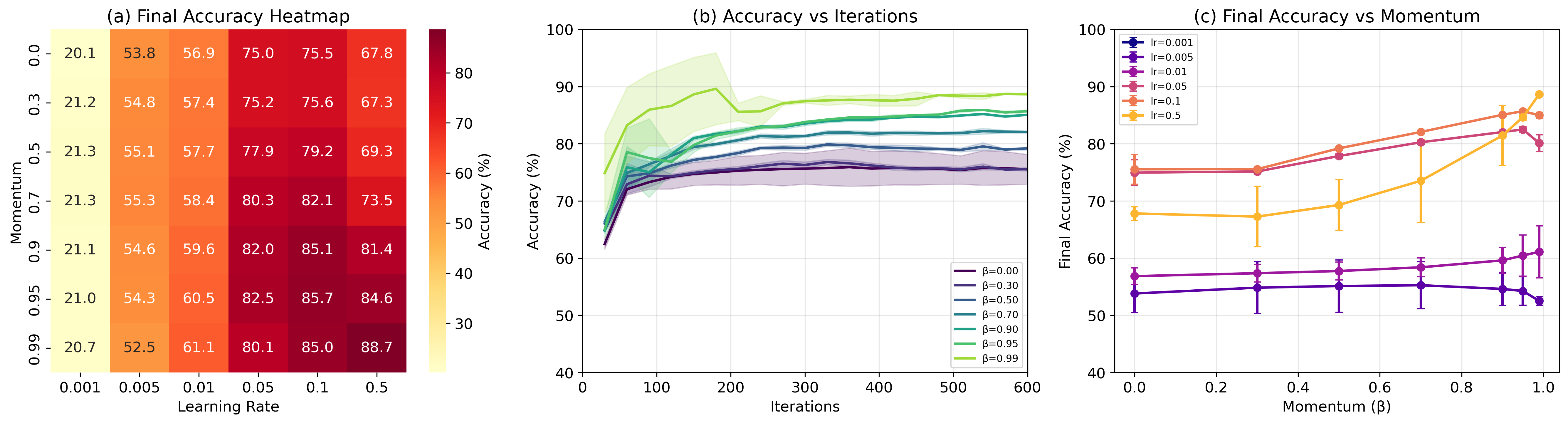}
    \caption{Ablation study on heterogeneous MNIST under Bit Flipping attack with RFA aggregation. (a) Final test accuracy heatmap across momentum ($\beta$) and learning rate ($\gamma_0$). (b) Training curves for each momentum value using its best learning rate. (c) Final accuracy vs momentum for different learning rates.}
    \label{fig:ablation}
\end{figure}

Figure~\ref{fig:ablation} shows the results. The key findings are:
\begin{itemize}
    \item \textbf{Momentum:} Any momentum parameter in the range $[0.9, 0.99]$ consistently outperforms momentum values less than $0.9$. Based on this observation, we adopt $\beta = 0.9$ as a standard choice throughout our experiments.
    \item \textbf{Learning rate:} The optimal learning rate varies across all momentum values and experimental configurations, requiring careful tuning. We observed that optimal learning rates for Baseline-Decay and \algname{Byz-NSGDM} are typically larger and in the range $[0.1, 1.0]$, while the Baseline method requires smaller learning rates to achieve stable convergence.
\end{itemize}

% \section{Conclusion}

\section{Conclusion}

In this work, we analyzed the behavior of Byzantine-robust normalized stochastic gradient descent with momentum (\algname{Byz-NSGDM}) under the generalized $(L_0,L_1)$-smoothness assumption. While normalized momentum combined with robust aggregation has been studied under classical $L$-smoothness and homogeneous settings, its theoretical guarantees under state-dependent smoothness and gradient heterogeneity were previously not understood. Our results provide a convergence analysis in this more general regime.

We established a non-asymptotic convergence rate of $O(K^{-1/4})$ up to an explicit Byzantine bias floor that depends on the robustness coefficient of the aggregation rule and the level of gradient heterogeneity. This characterization clarifies how generalized smoothness, heterogeneity, and adversarial corruption jointly affect the attainable accuracy.

Our empirical results on heterogeneous MNIST classification, synthetic $(L_0,L_1)$-smooth optimization, and character-level language modeling demonstrate that normalization combined with robust aggregation provides stable and competitive performance across diverse Byzantine attack strategies and aggregation rules. The experiments further highlight the practical robustness of the method across a wide range of hyperparameter choices.

Overall, our study bridges Byzantine-robust distributed optimization and optimization under generalized smoothness, and provides theoretical and empirical evidence that normalized momentum remains effective in adversarial and heterogeneous $(L_0,L_1)$-smooth settings. Future work includes extending the analysis to adaptive step-size schemes, tightening the bias floor characterization, and understanding optimal complexity limits under generalized smoothness and adversarial noise.

\section*{Code Availability}

Our code and experiments are publicly available at \url{https://github.com/armanbolatov/byz_nsgdm}.

\section*{LLM Usage Statement}

We used large language models (LLMs) to assist with code generation and visualization for experiments, to verify theorems and lemmas, to guide the proving process, and to polish the manuscript. The authors are fully responsible for the entire content and methodology of this paper.

\bibliography{refs}

\newpage

\appendix

% \section{Assumptions and Definition}

% The assumptions (\ref{as:lower_bound}--\ref{as:heterogeneity_detailed}) and the robustness definition (Definition~\ref{def:robust_agg}) are stated in the main paper for readability. We refer to those labels throughout this supplement.

% --- Technical Lemmas copied verbatim from draft ---
\section{Technical Lemmas}

\begin{lemma}[Lemma 1 from \citet{khiriat2024}]\label{lem:L0L1_extended}
Let $f_i$ be $(L_0,L_1)$-smooth for $i \in \cG$. Then the following properties hold:
\begin{enumerate}
    \item \textbf{Gradient Lipschitz bound:}
    \begin{equation}
        \|\nabla f_i(x) - \nabla f_i(y)\| \leq (L_0 + L_1\|\nabla f_i(y)\|)\exp(L_1\|x-y\|)\|x-y\|. \label{eq:grad_lip}
    \end{equation}
    
    \item \textbf{Function value bound:}
    \begin{equation*}
        f_i(y) \leq f_i(x) + \inp{\nabla f_i(x)}{y-x} + \frac{L_0 + L_1\|\nabla f_i(x)\|}{2}\exp(L_1\|x-y\|)\|y-x\|^2. 
    \end{equation*}
    
    \item \textbf{Gradient norm bound:}
    \begin{equation}
        \frac{\|\nabla f_i(x)\|^2}{4(L_0 + L_1\|\nabla f_i(x)\|)} \leq f_i(x) - f_i^*. \label{eq:grad_norm_bound}
    \end{equation}
    
    \item \textbf{Average function smoothness:}
    \begin{equation}
        f(y) \leq f(x) + \inp{\nabla f(x)}{y-x} + \frac{L_0 + \frac{L_1}{G}\sum_{i\in\cG}\|\nabla f_i(x)\|}{2}\exp(L_1\|x-y\|)\|y-x\|^2. \label{eq:avg_smooth}
    \end{equation}
\end{enumerate}
\end{lemma}

\begin{lemma}[Lemma 2 from \citet{khiriat2024}]\label{lem:sum_grad_norms}
Under Assumptions \ref{as:lower_bound} and \ref{as:L0L1_smoothness}, for any $x \in \R^d$:
\begin{equation}
    \frac{1}{G}\sum_{i\in\cG}\|\nabla f_i(x)\| \leq 8L_1(f(x) - f^*) + \frac{8L_1}{G}\sum_{i\in\cG}(f^* - f_i^*) + \frac{L_0}{L_1}. \label{eq:sum_of_grad_norms_bound}
\end{equation}
\end{lemma}

\begin{lemma}\label{lem:momentum_properties}
Let Assumptions \ref{as:L0L1_smoothness}, \ref{as:bounded_variance}, \ref{as:heterogeneity_detailed} hold.  
For each good worker $i\in\cG$ with momentum update

\[
v_i^{k}=(1-\eta)v_i^{k-1}+\eta\nabla f_{\xi_i^{k-1}}(x^{k-1}),
\]

consider normalized steps $\|x^{k}-x^{k-1}\|=\gamma$ (i.e., $\gamma_k\equiv\gamma>0$ in Algorithm~\ref{alg:Byz-NSGDM}) and a constant momentum parameter $\eta\in(0,1]$. Then:

\begin{enumerate}
\item \textbf{Bias.} For all $k\ge1$,

\[
\EE\left[v_i^{k}\big|x^{k-1}\right]=(1-\eta)v_i^{k-1}+\eta\nabla f_i(x^{k-1}).
\]

\item \textbf{Individual momentum tracking error.} For all $k\ge1$,
\begin{align}
\EE\big[\|v_i^{k}-\nabla f_i(x^{k-1})\|\big]
&\le (1-\eta)^{k}\EE\|v_i^0-\nabla f_i(x^0)\|
+\frac{2L_0\gamma e^{\gamma L_1}}{\eta} \notag\\
&\quad + 8L_1^2\gamma e^{\gamma L_1}\sum_{t=0}^{k-1}(1-\eta)^{k-1-t}\EE\big[f(x^t)-f^*\big] \notag\\
&\quad + \frac{8L_1^2\gamma e^{\gamma L_1}}{\eta}\big(f^*-f_i^*\big)
+\sqrt{\eta}\sigma.
\label{eq:vk_nablafxk_bound}
\end{align}

\item \textbf{Cross-worker momentum difference.} Let the expectation be over the stochastic gradients and an independent uniform draw of $i,j\in\cG$. Then, for all $k\ge1$,
\begin{align}
\EE\big[\|v_i^{k} - v_j^{k}\|\big]
&\le 2\zeta
+2(1-\eta)^{k} V_0
+2\sqrt{\eta}\sigma
+\frac{4L_0\gamma e^{\gamma L_1}}{\eta} \notag\\
&\quad + 16L_1^2\gamma e^{\gamma L_1}\sum_{t=0}^{k-1}(1-\eta)^{k-1-t}\EE\big[f(x^t)-f^*\big] \notag\\
&\quad + \frac{16L_1^2\gamma e^{\gamma L_1}}{\eta}\Delta^*,
\label{eq:cross_worker_diff}
\end{align}
where $V_0 \eqdef \tfrac{1}{G}\sum_{i\in\cG}\EE\|v_i^0-\nabla f_i(x^0)\|$ and $\Delta^*\eqdef \tfrac{1}{G}\sum_{i\in\cG}(f^*-f_i^*)$.
\end{enumerate}
\end{lemma}

\begin{proof}

The proof of these lemmas follows a very similar structure to that of the lemmas in \cite{khiriat2024}.

\textbf{Part 1.} This follows directly from the unbiased oracle in Assumption \ref{as:bounded_variance}:
$\EE[\nabla f_{\xi_i^{k-1}}(x^{k-1})\mid x^{k-1}]=\nabla f_i(x^{k-1})$.

\medskip
\noindent\textbf{Part 2.}
Define the tracking error $H_i^{k}\eqdef v_i^{k}-\nabla f_i(x^{k-1})$ and the auxiliary terms

\[
U_i^{k-1}\eqdef\nabla f_{\xi_i^{k-1}}(x^{k-1})-\nabla f_i(x^{k-1}),\qquad
G_i^{k}\eqdef\nabla f_i(x^{k-1})-\nabla f_i(x^{k}),
\]
where we also set $x^{-1} :=0$, yielding $G_i^{0} = 0$. From the momentum update $v_i^{k}=(1-\eta)v_i^{k-1}+\eta\nabla f_{\xi_i^{k-1}}(x^{k-1})$, we have
\begin{equation}\label{eq:Hi_recursion}
H_i^{k}=(1-\eta)H_i^{k-1}+(1-\eta)G_i^{k-1}+\eta U_i^{k-1}.
\end{equation}
Unrolling gives

\[
H_i^{k}=(1-\eta)^{k+1}H_i^{0}
+\sum_{t=0}^{k}(1-\eta)^{k+1-t}G_i^{t}
+\eta\sum_{t=0}^{k}(1-\eta)^{k-t}U_i^{t}.
\]

Taking norms, triangle inequality, and expectations,
\begin{align}
\EE\|H_i^{k}\|
&\le (1-\eta)^{k+1}\EE\|H_i^{0}\|
+\sum_{t=0}^{k}(1-\eta)^{k+1-t}\EE\|G_i^{t}\|
+\eta\EE\Big\|\sum_{t=0}^{k}(1-\eta)^{k-t}U_i^{t}\Big\|.
\label{eq:Hi_three_terms}
\end{align}

By Lemma~\ref{lem:L0L1_extended} (see \eqref{eq:grad_lip}) with $\|x^{t}-x^{t-1}\|=\gamma$,

\[
\|\nabla f_i(x^{t-1})-\nabla f_i(x^{t})\|
\le \big(L_0+L_1\|\nabla f_i(x^{t-1})\|\big)e^{\gamma L_1}\gamma.
\]

From Lemma~\ref{lem:L0L1_extended} (see \eqref{eq:grad_norm_bound}) and the same algebra as Lemma~2 in \citet{khiriat2024} applied to each $f_i$,
\(
\|\nabla f_i(x)\|
\le 8L_1\big(f_i(x)-f_i^*\big)+\frac{L_0}{L_1}
\le 8L_1\big(f(x)-f^*\big)+8L_1\big(f^*-f_i^*\big)+\frac{L_0}{L_1}.
\)
We obtain

\[
\EE\|G_i^{t}\|\le e^{\gamma L_1}\gamma\left(
2L_0+8L_1^2\EE[f(x^{t})-f^*]+8L_1^2(f^*-f_i^*)
\right).
\]

Summing the geometric weights $\sum_{t=0}^{k}(1-\eta)^{k+1-t}\le \eta^{-1}$ yields the three deterministic terms in \eqref{eq:vk_nablafxk_bound}.

By Jensen and Cauchy--Schwarz,

\[
\EE\Big\|\sum_{t=0}^{k}(1-\eta)^{k-t}U_i^{t}\Big\|
\le
\Big(\sum_{t=0}^{k}(1-\eta)^{2(k-t)}\EE\|U_i^{t}\|^2\Big)^{1/2}.
\]

The cross terms vanish since the oracle is unbiased and adapted to the filtration
$\{\mathcal F_t\}$: for $s<t$,

\[
\EE\big[\langle U_i^{s},U_i^{t}\rangle\big]
=\EE\big[\langle U_i^{s},\EE[U_i^{t}\mid\mathcal F_t]\rangle\big]=0.
\]

Using $\EE\|U_i^{t}\|^2\le\sigma^2$ and summing the geometric series gives
$\le \sigma/\sqrt{2\eta-\eta^2}\le \sigma/\sqrt{\eta}$, so the prefactor $\eta$ in
\eqref{eq:Hi_three_terms} produces the term $\sqrt{\eta}\sigma$.
This finishes Part~2.

\medskip
\noindent\textbf{Part 3.}
By the triangle inequality,
\(
\|v_i^{k}-v_j^{k}\|\le \|H_i^{k}\|+\|\nabla f_i(x^k)-\nabla f_j(x^k)\|+\|H_j^{k}\|.
\)
Taking expectations, using \eqref{eq:heterogeneity} and Jensen gives
$\EE\|\nabla f_i(x^k)-\nabla f_j(x^k)\|\le 2\zeta$; applying the bound of Part~2 to
$H_i^{k}$ and $H_j^{k}$ and averaging over $i,j$ replaces $(f^*-f_i^*)$ by
$\Delta^*$ and $\EE\|H_i^0\|$ by $V_0$, yielding \eqref{eq:cross_worker_diff}.
\end{proof}

\begin{lemma}[Average momentum error]\label{lem:avg_momentum_error}
Let $\bar v^{k} \eqdef \tfrac{1}{G}\sum_{i\in\cG} v_i^{k}$ and assume the setting of Lemma~\ref{lem:momentum_properties} with constant $\eta$ and normalized steps of size $\gamma$. Then for $k \ge 1$,
\begin{align}
A_{k} \eqdef \EE\big[\|\bar v^{k} - \nabla f(x^{k-1})\|\big]
&\le (1-\eta)^{k} V_0 + \frac{\sqrt{\eta}\sigma}{\sqrt{G}} + \frac{2L_0\gamma e^{\gamma L_1}}{\eta} \notag\\
&\quad + 8L_1^2\gamma e^{\gamma L_1}\sum_{t=0}^{k-1}(1-\eta)^{k-1-t}\EE[f(x^t)-f^*] \notag\\
 &\quad + \frac{8L_1^2\gamma e^{\gamma L_1}}{\eta}\Delta^*. \label{eq:Ak_bound2}
\end{align}
Here $V_0 \eqdef \tfrac{1}{G}\sum_{i\in\cG}\EE\|v_i^0-\nabla f_i(x^0)\|$ and $\Delta^* \eqdef \tfrac{1}{G}\sum_{i\in\cG}(f^*-f_i^*)$.
\end{lemma}

\begin{proof}
Let $H_i^{k}\eqdef v_i^{k}-\nabla f_i(x^{k-1})$ and define the averages

\[
\bar H^{k}\eqdef \frac1G\sum_{i\in\cG} H_i^{k},\
\bar G^{t}\eqdef \frac1G\sum_{i\in\cG}(\nabla f_i(x^{t-1})-\nabla f_i(x^{t})),\
\bar U^{t}\eqdef \frac1G\sum_{i\in\cG}\big(\nabla f_{\xi_i^{t}}(x^{t})-\nabla f_i(x^{t})\big).
\]

Averaging the per-worker recursion \eqref{eq:Hi_recursion} gives

\[
\bar H^{k}=(1-\eta)\bar H^{k-1}+(1-\eta)\bar G^{k-1}+\eta\bar U^{k-1}.
\]

Unrolling, taking norms, and expectations as in Lemma~\ref{lem:momentum_properties} (Part~2),
\begin{align*}
\EE\|\bar H^{k}\|
&\le (1-\eta)^{k}\EE\|\bar H^{0}\|
+\sum_{t=0}^{k-1}(1-\eta)^{k-t}\EE\|\bar G^{t}\|
+\eta\EE\Big\|\sum_{t=0}^{k-1}(1-\eta)^{k-1-t}\bar U^{t}\Big\|.
\end{align*}
For $\EE\|\bar G^{t}\|$, apply Lemma~\ref{lem:L0L1_extended} (see \eqref{eq:grad_lip}) and the same one-point bound used in Lemma~\ref{lem:momentum_properties} to get

\[
\EE\|\bar G^{t}\|\le e^{\gamma L_1}\gamma\left(2L_0+8L_1^2\EE[f(x^t)-f^*]+8L_1^2\Delta^*\right).
\]

For the noise term, independence across workers and unbiasedness give

\[
\EE\|\bar U^{t}\|^2
=\EE\Big\|\frac1G\sum_{i} U_i^{t}\Big\|^2
=\frac{1}{G^2}\sum_{i}\EE\|U_i^{t}\|^2\le \frac{\sigma^2}{G}.
\]

As in Part~2, cross terms in time vanish by the short argument
\[
\EE\big[\langle \bar U^{s},\bar U^{t}\rangle\big]
=\EE\big[\langle \bar U^{s},\EE[\bar U^{t}\mid\mathcal F_t]\rangle\big]=0\quad (s<t),
\]
so
\[
\EE\Big\|\sum_{t=0}^{k}(1-\eta)^{k-t}\bar U^{t}\Big\|
\le \frac{\sigma}{\sqrt{G}}\Big(\sum_{t=0}^{\infty}(1-\eta)^{2t}\Big)^{1/2}
\le \frac{\sigma}{\sqrt{G\eta}}.
\]
Putting these together and using $\sum_{t=0}^{k}(1-\eta)^{k+1-t}\le \eta^{-1}$ yields

\[
\EE\|\bar H^{k}\|
\le (1-\eta)^{k} V_0
+\frac{\sqrt{\eta}\sigma}{\sqrt{G}}
+\frac{2L_0\gamma e^{\gamma L_1}}{\eta}
+8L_1^2\gamma e^{\gamma L_1}\sum_{t=0}^{k-1}(1-\eta)^{k-1-t}\EE[f(x^t)-f^*]
+\frac{8L_1^2\gamma e^{\gamma L_1}}{\eta}\Delta^*.
\]

Finally, note that $\bar H^{k}=\bar v^{k}-\nabla f(x^{k-1})$, which gives \eqref{eq:Ak_bound2}.
\end{proof}

\begin{lemma}[Descent Lemma for \algname{Byz-NSGDM}]\label{lem:descent}
Let $f(x)=\tfrac1G\sum_{i\in\cG} f_i(x)$ with each $f_i$ $(L_0,L_1)$-smooth, and let the server update

\[
x^{k}=x^{k-1}-\gamma_{k-1}\frac{v^{k}}{\|v^{k}\|}\qquad(\gamma_{k-1}>0).
\]

Then
\begin{align*}
f(x^{k})
&\le f(x^{k-1})-\gamma_{k-1}\|\nabla f(x^{k-1})\|
+2\gamma_{k-1}\|\nabla f(x^{k-1})-v^{k}\|\\
&\quad+\frac{\gamma_{k-1}^2}{2}\exp(\gamma_{k-1} L_1)\left(L_0+\frac{L_1}{G}\sum_{i\in\cG}\|\nabla f_i(x^{k-1})\|\right).
\end{align*}
\end{lemma}

\begin{proof}
Apply Lemma~\ref{lem:L0L1_extended} (average smoothness, \eqref{eq:avg_smooth}) with
$x=x^{k-1}$ and $y=x^{k}=x^{k-1}-\gamma_{k-1}\tfrac{v^{k}}{\|v^{k}\|}$:
\begin{align}
f(x^{k})
&\le f(x^{k-1})+\inp{\nabla f(x^{k-1})}{-\gamma_{k-1}\tfrac{v^{k}}{\|v^{k}\|}}
+\frac{\gamma_{k-1}^2}{2}\exp(\gamma_{k-1} L_1)
\left(L_0+\frac{L_1}{G}\sum_{i\in\cG}\|\nabla f_i(x^{k-1})\|\right).
\label{eq:descent-start}
\end{align}
It remains to lower bound the linear term $\inp{\nabla f(x^{k-1})}{\tfrac{v^{k}}{\|v^{k}\|}}$.
Write
\begin{align*}
\inp{\nabla f(x^{k-1})}{\tfrac{v^{k}}{\|v^{k}\|}}
&=\inp{\nabla f(x^{k-1})-v^{k}}{\tfrac{v^{k}}{\|v^{k}\|}}
+\inp{v^{k}}{\tfrac{v^{k}}{\|v^{k}\|}}\\
&\ge -\|\nabla f(x^{k-1})-v^{k}\|+\|v^{k}\| \qquad\text{(Cauchy--Schwarz)}\\
&\ge -\|\nabla f(x^{k-1})-v^{k}\|+\|\nabla f(x^{k-1})\|-\|\nabla f(x^{k-1})-v^{k}\| \quad\text{(triangle ineq.)}\\
&=\|\nabla f(x^{k-1})\|-2\|\nabla f(x^{k-1})-v^{k}\|.
\end{align*}
Therefore

\[
-\gamma_{k-1}\inp{\nabla f(x^{k-1})}{\tfrac{v^{k}}{\|v^{k}\|}}
\le -\gamma_{k-1}\|\nabla f(x^{k-1})\|
+2\gamma_{k-1}\|\nabla f(x^{k-1})-v^{k}\|.
\]

Substituting this bound into \eqref{eq:descent-start} gives the claim.
\end{proof}

\begin{lemma}[Descent with curvature bound]\label{lem:descent_extended}
Let $\{x^k\}$ be generated by Algorithm~\ref{alg:Byz-NSGDM} and Assumption
\ref{as:L0L1_smoothness} hold. Then for $k \ge 1$,
\begin{align*}
f(x^{k})
&\le f(x^{k-1}) - \gamma_{k-1} \|\nabla f(x^{k-1})\|
    + 2\gamma_{k-1} \|\nabla f(x^{k-1})-v^{k}\|\\
&\quad + L_0\gamma_{k-1}^2 e^{\gamma_{k-1} L_1}
    + 4L_1^2\gamma_{k-1}^2 e^{\gamma_{k-1} L_1}\bigl(f(x^{k-1})-f^*\bigr)
    + 4L_1^2\gamma_{k-1}^2 e^{\gamma_{k-1} L_1}\Delta^*.
\end{align*}
\end{lemma}

\begin{proof}
Applying Lemma~\ref{lem:descent}, we have
\begin{align*}
f(x^{k})
&\le f(x^{k-1}) - \gamma_{k-1} \|\nabla f(x^{k-1})\|
   + 2\gamma_{k-1} \|\nabla f(x^{k-1})-v^{k}\|\\
   &\quad+ \frac{\gamma_{k-1}^2}{2} e^{\gamma_{k-1} L_1}
     \left( L_0 + \frac{L_1}{G}\sum_{i\in\cG}\|\nabla f_i(x^{k-1})\| \right).
\end{align*}
Using the sum-of-gradient-norms bound \eqref{eq:sum_of_grad_norms_bound} from
Lemma~\ref{lem:sum_grad_norms}, and substituting this into the curvature term yields
\begin{align*}
&\frac{\gamma_{k-1}^2}{2} e^{\gamma_{k-1} L_1}
\left( L_0 + \frac{L_1}{G}\sum_{i\in\cG}\|\nabla f_i(x^{k-1})\| \right)\\
&\quad\le
L_0\gamma_{k-1}^2 e^{\gamma_{k-1} L_1}
+4L_1^2\gamma_{k-1}^2 e^{\gamma_{k-1} L_1}\bigl(f(x^{k-1})-f^*\bigr)\\
&\qquad+4L_1^2\gamma_{k-1}^2 e^{\gamma_{k-1} L_1}\Delta^*.
\end{align*}
Combining gives the claim.
\end{proof}

\begin{lemma}[Byzantine Aggregation Error Decomposition]\label{lem:byz_agg_decomp_new}
Let $\bar{v}^{k} = \frac{1}{G}\sum_{i\in\mathcal{G}} v_i^{k}$ be the average momentum of good workers and $v^{k} = \texttt{ARAgg}(v_1^{k}, \ldots, v_n^{k})$ be the output of a $(\delta, \kappa)$-robust aggregator. Then for $k \ge 1$:
\begin{align*}
    \mathbb{E}[\|v^{k} - \nabla f(x^{k-1})\|] &\leq \mathbb{E}[\|\bar{v}^{k} - \nabla f(x^{k-1})\|] + \kappa \mathbb{E}_{i,j\in\mathcal{G}}[\|v_i^{k} - v_j^{k}\|], 
\end{align*}
where $\mathbb{E}_{i,j\in\mathcal{G}}[\|v_i^{k} - v_j^{k}\|] = \frac{1}{G^2}\sum_{i,j\in\mathcal{G}}\mathbb{E}[\|v_i^{k} - v_j^{k}\|]$.
\end{lemma}

\begin{proof}
Using the triangle inequality:
\begin{equation*}
    \|v^{k} - \nabla f(x^{k-1})\| \leq \|v^{k} - \bar{v}^{k}\| + \|\bar{v}^{k} - \nabla f(x^{k-1})\|.
\end{equation*}

Applying the $(\delta, \kappa)$-robust aggregator property:
\begin{equation*}
    \|v^{k} - \bar{v}^{k}\| \leq \frac{\kappa}{G} \sum_{i \in \mathcal{G}} \|v_i^{k} - \bar{v}^{k}\|.
\end{equation*}

For each $i \in \mathcal{G}$:
\begin{align*}
    \|v_i^{k} - \bar{v}^{k}\| &= \left\|v_i^{k} - \frac{1}{G}\sum_{j \in \mathcal{G}} v_j^{k}\right\| = \left\|\frac{1}{G}\sum_{j \in \mathcal{G}} (v_i^{k} - v_j^{k})\right\| \leq \frac{1}{G}\sum_{j \in \mathcal{G}} \|v_i^{k} - v_j^{k}\|.
\end{align*}

Therefore:
\begin{align*}
    \|v^{k} - \bar{v}^{k}\| &\leq \frac{\kappa}{G} \sum_{i \in \mathcal{G}} \frac{1}{G}\sum_{j \in \mathcal{G}} \|v_i^{k} - v_j^{k}\| = \frac{\kappa}{G^2} \sum_{i,j \in \mathcal{G}} \|v_i^{k} - v_j^{k}\|.
\end{align*}

Taking expectation and combining with the triangle inequality:
\begin{align*}
    \mathbb{E}[\|v^{k} - \nabla f(x^{k-1})\|] &\leq \mathbb{E}[\|v^{k} - \bar{v}^{k}\|] + \mathbb{E}[\|\bar{v}^{k} - \nabla f(x^{k-1})\|] \\
    &\leq \frac{\kappa}{G^2} \sum_{i,j \in \mathcal{G}} \mathbb{E}[\|v_i^{k} - v_j^{k}\|] + \mathbb{E}[\|\bar{v}^{k} - \nabla f(x^{k-1})\|] \\
    &= \kappa \mathbb{E}_{i,j\in\mathcal{G}}[\|v_i^{k} - v_j^{k}\|] + \mathbb{E}[\|\bar{v}^{k} - \nabla f(x^{k-1})\|].
\end{align*}
\end{proof}

\section{Proof of Theorem \ref{thm:byz_nsgdm_convergence}}

\begin{proof}
We begin by establishing the descent inequality and then perform a weighted telescoping analysis. Define the following quantities for notational convenience:
\begin{align*}
&\Delta_k = \mathbb{E}[f(x^k) - f^*], \quad A_{k} = \mathbb{E}[\|\bar{v}^{k} - \nabla f(x^{k-1})\|],\\
&B_{k} = \mathbb{E}[\|v^{k} - \nabla f(x^{k-1})\|], \quad D_{k} = \mathbb{E}_{i,j\in\mathcal{G}}[\|v_i^{k} - v_j^{k}\|].
\end{align*}

Since $\gamma \leq \frac{1}{2L_1}$ from our choice in \eqref{eq:gamma0_choice}, we have $\exp(L_1\gamma) \leq \exp(1/2) \leq 2$. Applying Lemma~\ref{lem:descent_extended}, we obtain the one-step descent inequality:
\begin{equation}\label{eq:descent_step}
\Delta_{k} \leq (1 + 8L_1^2\gamma^2)\Delta_{k-1} - \gamma\mathbb{E}[\|\nabla f(x^{k-1})\|] + 2\gamma B_{k} + 2L_0\gamma^2 + 8L_1^2\gamma^2\Delta^*.
\end{equation}

We now bound $B_{k}$ using Lemma~\ref{lem:byz_agg_decomp_new}: $B_{k} \leq A_{k} + \kappa D_{k}$. Therefore:
\begin{equation*}
2\gamma B_{k} \leq 2\gamma A_{k} + 2\gamma\kappa D_{k}.
\end{equation*}

Next, we bound $A_{k}$ and $D_{k}$ using Lemma~\ref{lem:avg_momentum_error} and equation \eqref{eq:cross_worker_diff}. Since $\exp(L_1\gamma) \leq 2$, we have:
\begin{align*}
A_{k} &\leq (1-\eta)^{k} V_0 + \frac{\sqrt{\eta}\sigma}{\sqrt{G}} + \frac{4L_0\gamma}{\eta} + 16L_1^2\gamma\sum_{t=0}^{k-1}(1-\eta)^{k-1-t}\Delta_t + \frac{16L_1^2\gamma}{\eta}\Delta^*, \\
D_{k} &\leq 2\zeta + 2(1-\eta)^{k} V_0 + 2\sqrt{\eta}\sigma + \frac{8L_0\gamma}{\eta} + 32L_1^2\gamma\sum_{t=0}^{k-1}(1-\eta)^{k-1-t}\Delta_t + \frac{32L_1^2\gamma}{\eta}\Delta^*. 
\end{align*}

Substituting these bounds into \eqref{eq:descent_step} and collecting terms yields:
\begin{align}
\Delta_{k} &\leq (1 + 8L_1^2\gamma^2)\Delta_{k-1} - \gamma\mathbb{E}[\|\nabla f(x^{k-1})\|] + 2\gamma A_{k} + 2\gamma\kappa D_{k} + 2L_0\gamma^2 + 8L_1^2\gamma^2\Delta^* \notag \\
&\leq (1 + 8L_1^2\gamma^2)\Delta_{k-1} - \gamma\mathbb{E}[\|\nabla f(x^{k-1})\|] + 32L_1^2\gamma^2(1 + 2\kappa)\sum_{t=0}^{k-1}(1-\eta)^{k-1-t}\Delta_t \notag \\
&\quad + 2\gamma(1 + 2\kappa)(1-\eta)^{k} V_0 + 4\gamma\kappa\zeta + \frac{2\gamma\sqrt{\eta}\sigma}{\sqrt{G}} + 4\gamma\kappa\sqrt{\eta}\sigma \notag \\
&\quad + 2L_0\gamma^2 + \frac{8(1 + 2\kappa)L_0\gamma^2}{\eta} + 8L_1^2\gamma^2\Delta^* + \frac{32(1 + 2\kappa)L_1^2\gamma^2}{\eta}\Delta^*. \label{eq:main_recursion}
\end{align}

To analyze the convergence over $K$ iterations, we introduce a weighted telescoping argument. Define the weights:
\begin{equation*}
w_k = \left(1 + 8L_1^2\gamma^2 + \frac{32(1 + 2\kappa)L_1^2\gamma^2}{\eta}\right)^{-k}, \quad S_K = \sum_{k=1}^{K}w_k.
\end{equation*}

Note that $w_k$ decreases geometrically with $k$. The key observation is that:
\begin{equation*}
w_{k-1} = w_k \cdot \left(1 + 8L_1^2\gamma^2 + \frac{32(1 + 2\kappa)L_1^2\gamma^2}{\eta}\right).
\end{equation*}

Now we multiply \eqref{eq:main_recursion} by $w_k$ and sum from $k = 1$ to $K$. Let's handle each term carefully. For the first term on the right-hand side:
\begin{align*}
\sum_{k=1}^{K}w_k(1 + 8L_1^2\gamma^2)\Delta_{k-1} &= \sum_{k=1}^{K}w_k\Delta_{k-1} + 8L_1^2\gamma^2\sum_{k=1}^{K}w_k\Delta_{k-1}.
\end{align*}

For the convolution term, we need to bound $\sum_{k=1}^{K}w_k\sum_{t=0}^{k-1}(1-\eta)^{k-1-t}\Delta_t$. We can rewrite this double sum by changing the order of summation. For a fixed $t$, the inner sum runs over all $k$ such that $t \leq k-1$, which means $k \geq t+1$. Therefore:
\begin{align*}
\sum_{k=1}^{K}w_k\sum_{t=0}^{k-1}(1-\eta)^{k-1-t}\Delta_t &= \sum_{t=0}^{K-1}\Delta_t \sum_{k=t+1}^{K}w_k(1-\eta)^{k-1-t}.
\end{align*}

Since $w_k$ is decreasing in $k$ (i.e., $w_{k+1} \leq w_k$), we can bound:
\begin{align*}
\sum_{k=t+1}^{K}w_k(1-\eta)^{k-1-t} &\leq w_{t+1} \sum_{k=t+1}^{K}(1-\eta)^{k-1-t} = w_{t+1} \sum_{j=0}^{K-t-1}(1-\eta)^{j} \\
&\leq w_{t+1} \sum_{j=0}^{\infty}(1-\eta)^{j} = \frac{w_{t+1}}{\eta}.
\end{align*}

Therefore:
\begin{equation*}
\sum_{k=1}^{K}w_k\sum_{t=0}^{k-1}(1-\eta)^{k-1-t}\Delta_t \leq \frac{1}{\eta}\sum_{t=0}^{K-1}w_{t+1}\Delta_t \leq \frac{1}{\eta}\sum_{k=1}^{K}w_k\Delta_{k-1}.
\end{equation*}

Now, summing \eqref{eq:main_recursion} multiplied by $w_k$ from $k = 1$ to $K$:
\begin{align*}
\sum_{k=1}^{K}w_k\Delta_{k} &\leq \sum_{k=1}^{K}w_k(1 + 8L_1^2\gamma^2)\Delta_{k-1} + \frac{32L_1^2\gamma^2(1 + 2\kappa)}{\eta}\sum_{k=1}^{K}w_k\Delta_{k-1} \\
&\quad - \gamma\sum_{k=1}^{K}w_k\mathbb{E}[\|\nabla f(x^{k-1})\|] + \text{(other terms)}.
\end{align*}

Combining the terms with $\Delta_{k-1}$:
\begin{align*}
\sum_{k=1}^{K}w_k\Delta_{k} &\leq \sum_{k=1}^{K}w_k\left(1 + 8L_1^2\gamma^2 + \frac{32(1 + 2\kappa)L_1^2\gamma^2}{\eta}\right)\Delta_{k-1} \\
&\quad - \gamma\sum_{k=1}^{K}w_k\mathbb{E}[\|\nabla f(x^{k-1})\|] + \text{(other terms)}.
\end{align*}

Using the relation $w_{k-1} = w_k(1 + 8L_1^2\gamma^2 + \frac{32(1 + 2\kappa)L_1^2\gamma^2}{\eta})$, we have:
% \begin{align}
% \sum_{k=1}^{K}w_k\left(1 + 8L_1^2\gamma^2 + \frac{32(1 + 2\kappa)L_1^2\gamma^2}{\eta}\right)\Delta_{k-1} = \sum_{k=1}^{K}w_{k-1}\Delta_{k-1} = \sum_{k=0}^{K-1}w_{k}\Delta_{k},
% \end{align}
% where we define $w_{0} = 1$ by convention.

% Now we apply the telescoping property. The left-hand side can be reindexed:
% \begin{align}
% \sum_{k=0}^{K}w_k\Delta_{k+1} = \sum_{k=1}^{K+1}w_{k-1}\Delta_k = \sum_{k=1}^{K}w_{k-1}\Delta_k + w_K\Delta_{K+1}.
% \end{align}

% Therefore:
% \begin{align}
% \sum_{k=1}^{K}w_{k-1}\Delta_k + w_K\Delta_{K+1} &\leq \sum_{k=0}^{K}w_{k-1}\Delta_k - \gamma\sum_{k=0}^{K}w_k\mathbb{E}[\|\nabla f(x^k)\|] + \text{(other terms)}.
% \end{align}

% This simplifies to:
\begin{align*}
\sum_{k=1}^{K}w_k\Delta_{k} &\leq \sum_{k=1}^{K}w_{k-1}\Delta_{k-1}  - \gamma\sum_{k=1}^{K}w_k\mathbb{E}[\|\nabla f(x^{k-1})\|] + \text{(other terms)},
\end{align*}
implying
\begin{align*}
\gamma\sum_{k=0}^{K}w_k\mathbb{E}[\|\nabla f(x^k)\|] &\leq \Delta_0 + \text{(other terms)},
\end{align*}
since $w_{0} = 1$ and both $w_K \geq 0$ and $\Delta_{K+1} \geq 0$ (by Assumption~\ref{as:lower_bound}).
% \begin{align}
% \Delta_0 &\geq \gamma\sum_{k=0}^{K}w_k\mathbb{E}[\|\nabla f(x^k)\|] - \text{(other terms)}.
% \end{align}

Let's now collect all the "other terms" from the weighted sum of \eqref{eq:main_recursion}. For the terms involving $(1-\eta)^k$:
\begin{align*}
\sum_{k=0}^{K}w_k \cdot 2\gamma(1 + 2\kappa)(1-\eta)^k V_0 &\leq 2\gamma(1 + 2\kappa)V_0 \sum_{k=0}^{K}w_k(1-\eta)^k \\
&\leq 2\gamma(1 + 2\kappa)V_0 \cdot w_0 \sum_{k=0}^{\infty}(1-\eta)^k = \frac{2\gamma(1 + 2\kappa)V_0}{\eta}.
\end{align*}

The remaining terms are constants multiplied by $\sum_{k=0}^{K}w_k = S_K$. Collecting everything:
\begin{align*}
\gamma\sum_{k=0}^{K}w_k\mathbb{E}[\|\nabla f(x^k)\|] &\leq \Delta_0 + \frac{2\gamma(1 + 2\kappa)V_0}{\eta} + 4\gamma\kappa\zeta S_K + \frac{2\gamma\sqrt{\eta}\sigma}{\sqrt{G}} S_K + 4\gamma\kappa\sqrt{\eta}\sigma S_K \\
&\quad + \left(2L_0\gamma^2 + \frac{8(1 + 2\kappa)L_0\gamma^2}{\eta}\right)S_K \\
&\quad + \left(8L_1^2\gamma^2 + \frac{32(1 + 2\kappa)L_1^2\gamma^2}{\eta}\right)\Delta^* S_K.
\end{align*}

Since $\sum_{k=0}^{K}w_k\mathbb{E}[\|\nabla f(x^k)\|] \geq S_K \cdot \min_{0 \leq k \leq K}\mathbb{E}[\|\nabla f(x^k)\|]$, we obtain:
\begin{align}\label{eq:key_inequality}
\gamma S_K \min_{0 \leq k \leq K}\mathbb{E}[\|\nabla f(x^k)\|] &\leq \Delta_0 + \frac{2\gamma(1 + 2\kappa)V_0}{\eta} + 4\gamma\kappa\zeta S_K \notag \\
&\quad + S_K\left(2L_0\gamma^2 + \frac{8(1 + 2\kappa)L_0\gamma^2}{\eta}\right) \notag \\
&\quad + S_K\left(8L_1^2\gamma^2 + \frac{32(1 + 2\kappa)L_1^2\gamma^2}{\eta}\right)\Delta^* \notag \\
&\quad + \frac{2\gamma\sqrt{\eta}\sigma S_K}{\sqrt{G}} + 4\gamma\kappa\sqrt{\eta}\sigma S_K.
\end{align}

To complete the proof, we need a lower bound on $S_K$. Using the inequality $(1+x)^n \leq e^{nx}$ for $x \geq 0$:
\begin{align*}
S_K = \sum_{k=0}^{K}w_k &= \sum_{k=0}^{K}\left(1 + 8L_1^2\gamma^2 + \frac{32(1 + 2\kappa)L_1^2\gamma^2}{\eta}\right)^{-(k+1)} \\
&\geq \frac{K+1}{\left(1 + 8L_1^2\gamma^2 + \frac{32(1 + 2\kappa)L_1^2\gamma^2}{\eta}\right)^{K+1}} \\
&\geq \frac{K+1}{\exp\left(\left(8L_1^2\gamma^2 + \frac{32(1 + 2\kappa)L_1^2\gamma^2}{\eta}\right)(K+1)\right)}.
\end{align*}

Substituting $\gamma = \gamma_0/(K+1)^{3/4}$ and $\eta = (K+1)^{-1/2}$:
\begin{align*}
8L_1^2\gamma^2 + \frac{32(1 + 2\kappa)L_1^2\gamma^2}{\eta} &= \frac{8L_1^2\gamma_0^2}{(K+1)^{3/2}} + \frac{32(1 + 2\kappa)L_1^2\gamma_0^2}{(K+1)^{3/2}} \cdot (K+1)^{1/2} \\
&= \frac{8L_1^2\gamma_0^2}{(K+1)^{3/2}} + \frac{32(1 + 2\kappa)L_1^2\gamma_0^2}{K+1}.
\end{align*}

Therefore:
\begin{align*}
S_K &\geq \frac{K+1}{\exp\left(\frac{8L_1^2\gamma_0^2}{(K+1)^{1/2}} + 32(1 + 2\kappa)L_1^2\gamma_0^2\right)}.
\end{align*}

The choice \eqref{eq:gamma0_choice} ensures that $32(1 + 2\kappa)L_1^2\gamma_0^2 \leq 1/4$ and $\frac{8L_1^2\gamma_0^2}{(K+1)^{1/2}} \leq 1/4$ for all $K \geq 0$. Hence:
\begin{equation*}
S_K \geq (K+1)e^{-1/2}.
\end{equation*}

Finally, dividing \eqref{eq:key_inequality} by $\gamma S_K$ and substituting the parameter choices:
\begin{align*}
\min_{0 \leq k \leq K}\mathbb{E}[\|\nabla f(x^k)\|] &\leq \frac{\Delta_0}{\gamma S_K} + \frac{2(1 + 2\kappa)V_0}{\eta S_K} + 4\kappa\zeta \\
&\quad + 2L_0\gamma + \frac{8(1 + 2\kappa)L_0\gamma}{\eta} + 8L_1^2\gamma\Delta^* + \frac{32(1 + 2\kappa)L_1^2\gamma\Delta^*}{\eta} \\
&\quad + \frac{2\sqrt{\eta}\sigma}{\sqrt{G}} + 4\kappa\sqrt{\eta}\sigma.
\end{align*}

Substituting $\gamma = \gamma_0/(K+1)^{3/4}$, $\eta = (K+1)^{-1/2}$, $S_K \geq (K+1)e^{-1/2}$, and $\Delta_0 = f(x^0) - f^*$:
\begin{align*}
\min_{0 \leq k \leq K}\mathbb{E}[\|\nabla f(x^k)\|] &\leq \frac{e^{1/2}(f(x^0) - f^*)(K+1)^{3/4}}{\gamma_0(K+1)} + \frac{2e^{1/2}(1 + 2\kappa)V_0(K+1)^{1/2}}{K+1} + 4\kappa\zeta \\
&\quad + \frac{2L_0\gamma_0}{(K+1)^{3/4}} + \frac{8(1 + 2\kappa)L_0\gamma_0(K+1)^{1/2}}{(K+1)^{3/4}} \\
&\quad + \frac{8L_1^2\Delta^*\gamma_0}{(K+1)^{3/4}} + \frac{32(1 + 2\kappa)L_1^2\Delta^*\gamma_0(K+1)^{1/2}}{(K+1)^{3/4}} \\
&\quad + \frac{2\sigma}{\sqrt{G}(K+1)^{1/4}} + \frac{4\kappa\sigma}{(K+1)^{1/4}}.
\end{align*}

Simplifying the powers of $(K+1)$ yields:
\begin{align*}
\min_{0 \leq k \leq K}\mathbb{E}[\|\nabla f(x^k)\|] &\leq \frac{e^{1/2}(f(x^0) - f^*)}{\gamma_0(K+1)^{1/4}} + \frac{2e^{1/2}(1 + 2\kappa)V_0}{(K+1)^{1/2}} + 4\kappa\zeta \\
&\quad + \frac{2L_0\gamma_0}{(K+1)^{3/4}} + \frac{8(1 + 2\kappa)L_0\gamma_0}{(K+1)^{1/4}} \\
&\quad + \frac{8L_1^2\Delta^*\gamma_0}{(K+1)^{3/4}} + \frac{32(1 + 2\kappa)L_1^2\Delta^*\gamma_0}{(K+1)^{1/4}} \\
&\quad + \frac{2\sigma}{\sqrt{G}(K+1)^{1/4}} + \frac{4\kappa\sigma}{(K+1)^{1/4}}.
\end{align*}
\end{proof}

\section{Byzantine-Robust Aggregators under New Robustness Definition}\label{appendix:byz-robust-aggregators}

This section provides proofs that several aggregation rules from \cite{allouah2023} satisfy our modified $(\delta, \kappa)$-robustness definition (Definition~\ref{def:robust_agg}), which removes the squares compared to the original definition in \cite{allouah2023}. We prove robustness for Geometric Median (GM) and Coordinate-wise Median (CWMed). The proofs follow a similar structure to those in \cite{allouah2023} but are adapted for the linear (non-squared) robustness condition.

\subsection{Nearest Neighbor Mixing (NNM) and Robustness}

\paragraph{NNM operator.}
Given $x_1,\ldots,x_n\in\R^d$ and an integer $B<n/2$, let $\cG$ be the set of good indices with $|\cG|=G=n-B$.
For a pivot $\mu\in\R^d$, denote by $\mathcal N_\mu$ the indices of its $G$ nearest neighbors among $\{1,\ldots,n\}$ and set

\[
y_\mu \eqdef \frac{1}{G}\sum_{j\in\mathcal N_\mu} x_j.
\]

For each $i\in[n]$, define $y_i\eqdef y_{x_i}$. For the good set $\cG$ let

\[
\bar x_{\cG}\eqdef\frac{1}{G}\sum_{j\in\cG}x_j,\qquad
\bar y_{\cG}\eqdef\frac{1}{G}\sum_{i\in\cG}y_i.
\]

\begin{assumption}[Bounded leverage of the good cluster]\label{as:BL}
There exists a constant $C>1$ such that

\[
\max_{j\in\cG}\|x_j-\bar x_{\cG}\|
\le
\frac{C}{G}\sum_{t\in\cG}\|x_t-\bar x_{\cG}\|.
\]

\end{assumption}

\begin{remark}
Assumption~\ref{as:BL} always holds with $C = G$ since $\max_{j\in\cG}\|x_j-\bar x_{\cG}\| \le \sum_{t\in\cG}\|x_t-\bar x_{\cG}\|$. However, such a choice leads to a large robustness coefficient $\kappa'$ in the subsequent analysis. In practice, when the good workers are well-clustered, $C$ can be much smaller than $G$.
\end{remark}

For a good pivot $i\in\cG$ we write
\[
\alpha_i \eqdef \frac{|\mathcal N_{x_i}\setminus\cG|}{G}=\frac{|\cG\setminus\mathcal N_{x_i}|}{G}\in[0,1/2).
\]
Since at most $B$ Byzantine points exist overall, we always have
\begin{equation*}
\alpha_i \le \frac{B}{G}=\frac{B}{n-B}.
\end{equation*}

\begin{lemma}[NNM Pairing Property]\label{lem:nnm-pairing}
For any pivot $\mu\in\R^d$, let $\mathcal N_\mu$ be the set of indices of the $G$ nearest neighbors to $\mu$. Then
\begin{equation}\label{eq:sum-pairing}
\sum_{b\in\mathcal N_\mu\setminus\cG}\|x_b-\mu\| \le \sum_{g\in\cG\setminus\mathcal N_\mu}\|x_g-\mu\|.
\end{equation}
\end{lemma}

\begin{proof}
Since $|\mathcal N_\mu|=|\cG|=G$, we have $|\mathcal N_\mu\setminus\cG|=|\cG\setminus\mathcal N_\mu|$. Let $m=|\mathcal N_\mu\setminus\cG|$.

Order the points in $\mathcal N_\mu\setminus\cG$ as $b_1,\ldots,b_m$ and points in $\cG\setminus\mathcal N_\mu$ as $g_1,\ldots,g_m$. By the definition of NNM, every point $x_{b_j}$ in $\mathcal N_\mu$ is at least as close to $\mu$ as every point $x_{g_j}$ not in $\mathcal N_\mu$. Therefore, for each $j\in[m]$:

\[
\|x_{b_j}-\mu\| \le \|x_{g_j}-\mu\|.
\]

Summing over all $j$ yields \eqref{eq:sum-pairing}.
\end{proof}

\begin{lemma}[General Pivot Bound]\label{lem:general-pivot}
For any pivot $\mu\in\R^d$,
\[
\|y_\mu-\bar x_{\cG}\| \le \frac{2}{G}\sum_{j\in\cG\setminus\mathcal N_\mu}\|x_j-\mu\|.
\]
\end{lemma}

\begin{proof}
Write

\[
y_\mu-\bar x_{\cG}
= \frac{1}{G}\left(\sum_{j\in\mathcal N_{\mu}}x_j - \sum_{j\in\cG}x_{j}\right) = \frac{1}{G}\left(\sum_{j\in\mathcal N_{\mu}\setminus\cG}x_j - \sum_{j\in\cG\setminus\mathcal N_{\mu}}x_{j}\right).
\]

Since $|\mathcal N_\mu\setminus\cG|=|\cG\setminus\mathcal N_\mu|$, we can subtract and add $\mu$ to obtain

\[
y_\mu-\bar x_{\cG} = \frac{1}{G}\left(\sum_{j\in\mathcal N_{\mu}\setminus\cG}(x_j-\mu) - \sum_{j\in\cG\setminus\mathcal N_{\mu}}(x_j-\mu)\right).
\]

Taking norms and applying the triangle inequality:

\[
\|y_\mu-\bar x_{\cG}\| \le \frac{1}{G}\left(\sum_{j\in\mathcal N_{\mu}\setminus\cG}\|x_j-\mu\| + \sum_{j\in\cG\setminus\mathcal N_{\mu}}\|x_j-\mu\|\right).
\]

Applying Lemma~\ref{lem:nnm-pairing}:
\[
\|y_\mu-\bar x_{\cG}\| \le \frac{2}{G}\sum_{j\in\cG\setminus\mathcal N_{\mu}}\|x_j-\mu\|.
\]
\end{proof}

\begin{lemma}\label{lem:avg-good-pivot}
Under Assumption~\ref{as:BL}, for any good pivot $i\in\cG$, letting $\alpha\eqdef\max_{i\in\cG}\alpha_i\le B/(n-B)$,

\[
\frac{1}{G}\sum_{i\in\cG}\|y_i-\bar x_{\cG}\| \le \frac{4\alpha C}{G}\sum_{t\in\cG}\|x_t-\bar x_{\cG}\|.
\]

\end{lemma}

\begin{proof}
By Lemma~\ref{lem:general-pivot} with $\mu=x_i$:

\[
\|y_i-\bar x_{\cG}\| \le \frac{2}{G}\sum_{j\in\cG\setminus\mathcal N_{x_i}}\|x_j-x_i\|.
\]

Applying the triangle inequality $\|x_j-x_i\|\le\|x_j-\bar x_{\cG}\|+\|x_i-\bar x_{\cG}\|$, $|\cG\setminus\mathcal N_{x_i}|=\alpha_i G$ and by Assumption~\ref{as:BL}:

\begin{align*}
\|y_i-\bar x_{\cG}\| &\le \frac{2}{G}\sum_{j\in\cG\setminus\mathcal N_{x_i}}\left(\|x_j-\bar x_{\cG}\|+\|x_i-\bar x_{\cG}\|\right) \\
&\le \frac{4 |\cG\setminus\mathcal N_{x_i}|}{G}  \max_{j \in \cG\setminus\mathcal N_{x_i}} \|x_j-\bar x_{\cG} \| \le \frac{4 \alpha_i C}{G} \sum_{j \in \cG} \|x_j-\bar x_{\cG} \|.
\end{align*}

Averaging over $i \in \cG$ and applying $\alpha_i \le \alpha$ we have the desired lemma.
\end{proof}

\begin{lemma}\label{lem:center-disp}
Under Assumption~\ref{as:BL}, letting $\alpha\eqdef \max_{i\in\cG}\alpha_i\le B/(n-B)$, we have
\begin{align}
\|\bar y_{\cG}-\bar x_{\cG}\|
&\le \frac{4\alpha C}{G}\sum_{t\in\cG}\|x_t-\bar x_{\cG}\|, \label{eq:center-shift}\\
\frac{1}{G}\sum_{i\in\cG}\|y_i-\bar y_{\cG}\|
&\le \frac{8\alpha C}{G}\sum_{t\in\cG}\|x_t-\bar x_{\cG}\|. \label{eq:dispersion}
\end{align}
\end{lemma}

\begin{proof}
For \eqref{eq:center-shift}, since $\bar y_{\cG}=\frac{1}{G}\sum_{i\in\cG}y_i$, by triangle inequality:

\[
\|\bar y_{\cG}-\bar x_{\cG}\| = \left\|\frac{1}{G}\sum_{i\in\cG}(y_i-\bar x_{\cG})\right\| \le \frac{1}{G}\sum_{i\in\cG}\|y_i-\bar x_{\cG}\|.
\]

Applying Lemma~\ref{lem:avg-good-pivot} yields \eqref{eq:center-shift}.

For \eqref{eq:dispersion}, by the triangle inequality:

\[
\|y_i-\bar y_{\cG}\| \le \|y_i-\bar x_{\cG}\| + \|\bar y_{\cG}-\bar x_{\cG}\|.
\]

Averaging over $i\in\cG$ and applying Lemma~\ref{lem:avg-good-pivot} and \eqref{eq:center-shift}:

\[
\frac{1}{G}\sum_{i\in\cG}\|y_i-\bar y_{\cG}\| \le \frac{4\alpha C}{G}\sum_{t\in\cG}\|x_t-\bar x_{\cG}\| + \frac{4\alpha C}{G}\sum_{t\in\cG}\|x_t-\bar x_{\cG}\| = \frac{8\alpha C}{G}\sum_{t\in\cG}\|x_t-\bar x_{\cG}\|.
\]

\end{proof}

\begin{lemma}[Robustness after NNM]\label{lem:nnm-linear}
Let $F:\R^{d\times n}\to\R^d$ be $(\delta,\kappa)$-robust in the sense of Definition~\ref{def:robust_agg}.
Under Assumption~\ref{as:BL}, for any $x_1,\ldots,x_n$,

\[
\big\|F\big(\texttt{NNM}(x_1,\ldots,x_n)\big)-\bar x_{\cG}\big\|
\le
\frac{\kappa'}{G}\sum_{t\in\cG}\|x_t-\bar x_{\cG}\|,
\]

where $\kappa'=(8\kappa+4)\alpha C$ and $\alpha\le B/(n-B)$.
\end{lemma}

\begin{proof}
Let $(y_1,\ldots,y_n)=\texttt{NNM}(x_1,\ldots,x_n)$. By the triangle inequality:

\[
\|F(y_1,\ldots,y_n)-\bar x_{\cG}\|
\le \|F(y_1,\ldots,y_n)-\bar y_{\cG}\|+\|\bar y_{\cG}-\bar x_{\cG}\|.
\]

By the $(\delta,\kappa)$-robustness of $F$ (Definition~\ref{def:robust_agg}):

\[
\|F(y_1,\ldots,y_n)-\bar y_{\cG}\| \le \frac{\kappa}{G}\sum_{i\in\cG}\|y_i-\bar y_{\cG}\|.
\]

Applying Lemma~\ref{lem:center-disp}:
\begin{align*}
\|F(y_1,\ldots,y_n)-\bar x_{\cG}\|
&\le \frac{\kappa}{G}\sum_{i\in\cG}\|y_i-\bar y_{\cG}\| + \|\bar y_{\cG}-\bar x_{\cG}\| \\
&\le \frac{8\kappa\alpha C}{G}\sum_{t\in\cG}\|x_t-\bar x_{\cG}\| + \frac{4\alpha C}{G}\sum_{t\in\cG}\|x_t-\bar x_{\cG}\| \\
&= \frac{(8\kappa+4)\alpha C}{G}\sum_{t\in\cG}\|x_t-\bar x_{\cG}\|.
\end{align*}
\end{proof}

\begin{remark}
The composition of NNM with a $(\delta,\kappa)$-robust aggregator $F$ yields a robust aggregator with constant $\kappa'=(8\kappa+4)\alpha C$, where $\alpha\le B/(n-B)$ measures the fraction of Byzantine contamination. When $B\ll n$, this additional factor is small, showing that NNM preprocessing preserves robustness up to a multiplicative constant depending on the Byzantine fraction and the leverage constant $C$.
\end{remark}

\subsection{Geometric Median}

The Geometric Median of $v_1, \ldots, v_n \in \R^d$, denoted by $\text{GM}(v_1, \ldots, v_n)$, is defined as a vector that minimizes the sum of the $\ell_2$-distances to these vectors:
\begin{equation*}
    \text{GM}(v_1, \ldots, v_n) \in \argmin_{y \in \R^d} \sum_{i=1}^n \|y - v_i\|.
\end{equation*}

\begin{proposition}[Geometric Median Robustness]\label{prop:gm_new}
Let $n \in \mathbb{N}$ and $\delta < \frac{1}{2}$ such that $B = \lfloor \delta n \rfloor < \frac{n}{2}$. 
The Geometric Median is $(\delta, \kappa)$-robust with 
\begin{equation*}
    \kappa = 2\left(1 + \frac{B}{n-2B}\right).
\end{equation*}
\end{proposition}

\begin{proof}
The proof follows the same approach as in \cite{allouah2023} but stops before squaring both sides to match our linear robustness definition.

Let $v^* := \text{GM}(v_1, \ldots, v_n)$ be the geometric median of the input vectors. Consider any subset $S \subseteq [n]$ of size $|S| = \lfloor(1-\delta)n\rfloor = n - B$. By the reverse triangle inequality, for any $i \in S$, we have
\begin{equation*}
    \|v^* - v_i\| \geq |\|v^* - \bar{v}_S\| - \|v_i - \bar{v}_S\||.
\end{equation*}

Similarly, for any $i \in [n] \setminus S$, we obtain
\begin{equation*}
    \|v^* - v_i\| \geq |\|v_i - \bar{v}_S\| - \|v^* - \bar{v}_S\||.
\end{equation*}

Summing over all input vectors and using the fact that $|[n] \setminus S| \leq B$:
\begin{align*}
    \sum_{i \in [n]} \|v^* - v_i\| &\geq \sum_{i \in S} (|\|v^* - \bar{v}_S\| - \|v_i - \bar{v}_S\||) + \sum_{i \in [n]\setminus S} (|\|v_i - \bar{v}_S\| - \|v^* - \bar{v}_S\||) \\
    &\geq (n-2B) \|v^* - \bar{v}_S\| - \sum_{i \in S} \|v_i - \bar{v}_S\| + \sum_{i \in [n]\setminus S} \|v_i - \bar{v}_S\|.
\end{align*}

Rearranging terms:
\begin{equation*}
    \|v^* - \bar{v}_S\| \leq \frac{1}{n-2B} \left( \sum_{i \in [n]} \|v^* - v_i\| + \sum_{i \in S} \|v_i - \bar{v}_S\| - \sum_{i \in [n]\setminus S} \|v_i - \bar{v}_S\| \right).
\end{equation*}

By the definition of the geometric median:
\begin{equation*}
    \sum_{i \in [n]} \|v^* - v_i\| \leq \sum_{i \in [n]} \|\bar{v}_S - v_i\|.
\end{equation*}

Therefore:
\begin{align*}
    \|v^* - \bar{v}_S\| &\leq \frac{1}{n-2B} \left( \sum_{i \in [n]} \|v_i - \bar{v}_S\| + \sum_{i \in S} \|v_i - \bar{v}_S\| - \sum_{i \in [n]\setminus S} \|v_i - \bar{v}_S\| \right) \\
    &= \frac{1}{n-2B} \left( 2\sum_{i \in S} \|v_i - \bar{v}_S\| \right) = \frac{2}{n-2B} \sum_{i \in S} \|v_i - \bar{v}_S\| \\
    &= \frac{2(n-B)}{(n-2B)|S|} \sum_{i \in S} \|v_i - \bar{v}_S\| = 2\left(1 + \frac{B}{n-2B}\right) \frac{1}{|S|} \sum_{i \in S} \|v_i - \bar{v}_S\|.
\end{align*}

This establishes the desired robustness property with $\kappa = 2\left(1 + \frac{B}{n-2B}\right)$.
\end{proof}

\subsection{Coordinate-wise Median}

For input vectors $v_1, \ldots, v_n \in \R^d$, their coordinate-wise median, denoted by $\text{CWMed}(v_1, \ldots, v_n)$, is defined as a vector whose $k$-th coordinate is:
\begin{equation*}
    \left[\text{CWMed}(v_1, \ldots, v_n)\right]_k := \text{Median}\left( [v_1]_k, \ldots, [v_n]_k \right).
\end{equation*}

\begin{lemma}[Coordinate-wise Robustness]\label{lem:cw-kappa-new}
Assume that $F: \R^{d \times n} \to \R^d$ is a coordinate-wise aggregation function, i.e., there exist $d$ real-valued functions $F_1,\ldots,F_d : \R^n \to \R$ such that for all $v_1,\ldots,v_n \in \R^d$, 
$\left[F(v_1,\ldots,v_n) \right]_k = F_k([v_1]_k,\ldots,[v_n]_k)$.
If for each $k \in [d]$, $F_k$ is $(\delta,\kappa)$-robust (in the sense of Definition~\ref{def:robust_agg} applied to scalars), then $F$ is $(\delta,\kappa\sqrt{d})$-robust.
\end{lemma}

\begin{proof}
Since each $F_k$ is $(\delta,\kappa)$-robust, for any subset $\mathcal{G}$ of good workers with $|\mathcal{G}| = G \ge (1-\delta)n$:
\[
|[F(v_1,\ldots,v_n)]_k - [\bar{v}_{\mathcal{G}}]_k| \le \frac{\kappa}{G}\sum_{i\in\mathcal{G}}|[v_i]_k - [\bar{v}_{\mathcal{G}}]_k|.
\]
Summing over all coordinates:
\[
\sum_{k=1}^d |[F(v)]_k - [\bar{v}_{\mathcal{G}}]_k| \le \frac{\kappa}{G}\sum_{i\in\mathcal{G}}\sum_{k=1}^d|[v_i]_k - [\bar{v}_{\mathcal{G}}]_k| = \frac{\kappa}{G}\sum_{i\in\mathcal{G}}\|v_i - \bar{v}_{\mathcal{G}}\|_1.
\]
Using $\|a\|_2 \le \|a\|_1$ on the left and $\|a\|_1 \le \sqrt{d}\|a\|_2$ on the right:
\[
\|F(v) - \bar{v}_{\mathcal{G}}\| \le \frac{\kappa\sqrt{d}}{G}\sum_{i\in\mathcal{G}}\|v_i - \bar{v}_{\mathcal{G}}\|.
\]
\end{proof}
\begin{proposition}[Coordinate-wise Median Robustness]\label{prop:cwmed_new}
Let $n \in \mathbb{N}$ and $\delta < \frac{1}{2}$ such that $B = \lfloor \delta n \rfloor < \frac{n}{2}$. 
The Coordinate-wise Median is $(\delta, \kappa)$-robust with 
\begin{equation*}
    \kappa = 2\sqrt{d}\left(1 + \frac{B}{n-2B}\right),
\end{equation*}
where $d$ is the dimension of the input vectors.
\end{proposition}

\begin{proof}
Since geometric median coincides with median for one-dimensional inputs, it follows from Proposition~\ref{prop:gm_new} that for one-dimensional inputs, median is $(\delta,\kappa_0)$-robust with $\kappa_0 = 2\left(1 + \frac{B}{n-2B}\right)$. Since coordinate-wise median is a coordinate-wise aggregation function, we apply Lemma~\ref{lem:cw-kappa-new} to obtain that the $d$-dimensional coordinate-wise median is $(\delta, \kappa_0\sqrt{d})$-robust, giving the stated bound.
\end{proof}

\begin{remark}
The $\sqrt{d}$ factor arises from converting between $\ell_1$ and $\ell_2$ norms when combining per-coordinate bounds into a vector bound. This dimension dependence is a known limitation of coordinate-wise aggregation rules under linear robustness definitions. In contrast, the geometric median (Proposition~\ref{prop:gm_new}) achieves dimension-independent robustness because it operates directly on the full vectors.
\end{remark}

\newpage
\section{Training Curves}\label{appendix:curves}

\begin{figure}[H]
    \centering
    \includegraphics[width=\textwidth]{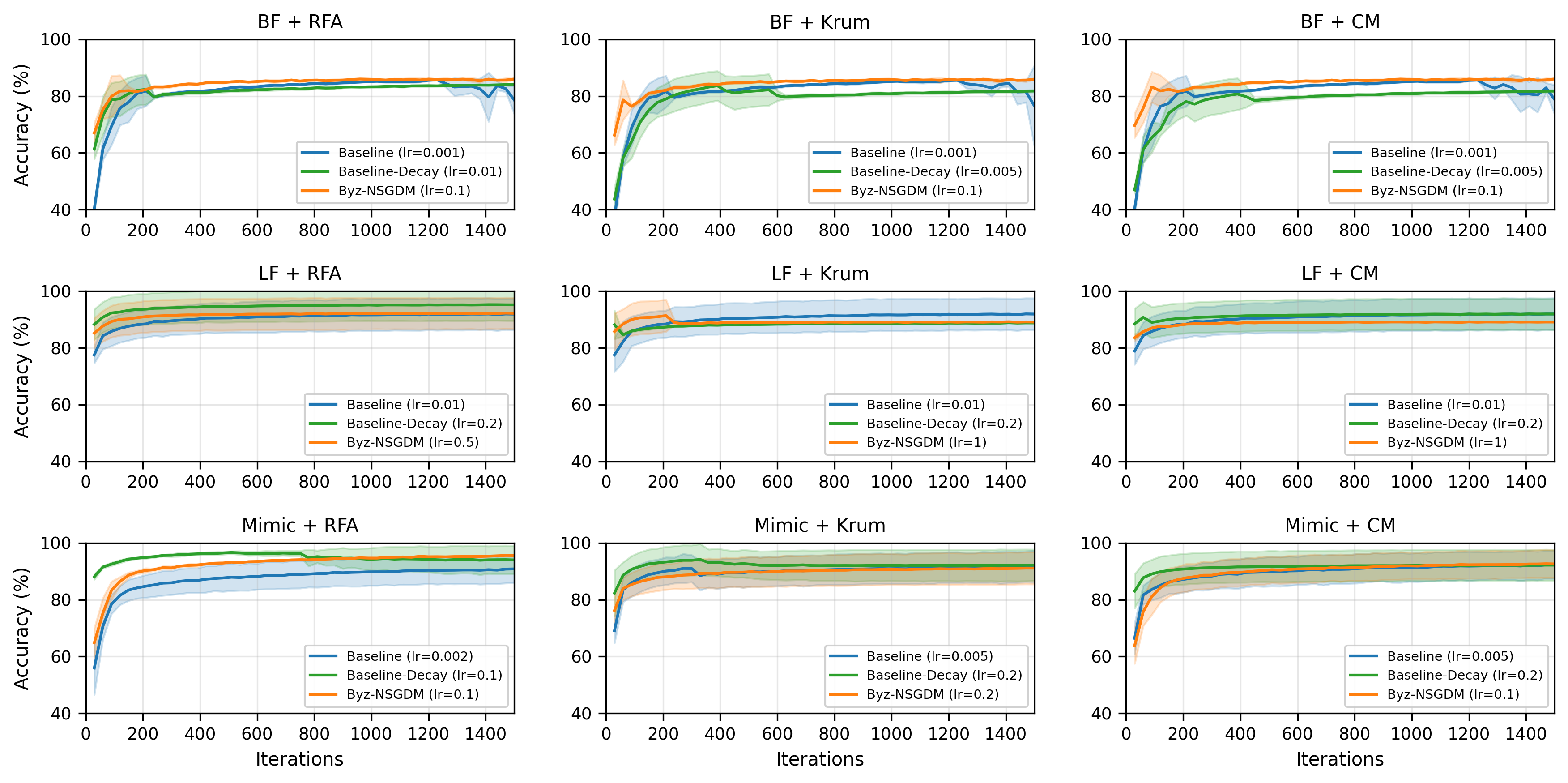}
    \caption{Test accuracy curves on heterogeneous MNIST under Byzantine attacks. Each row corresponds to a different attack (BF, LF, Mimic), and each column to a different aggregator (RFA, Krum, CM). Lines show mean accuracy over 3 seeds with shaded regions indicating $\pm 1$ standard deviation.}
    \label{fig:mnist_curves}
\end{figure}

\begin{figure}[H]
    \centering
    \includegraphics[width=\textwidth]{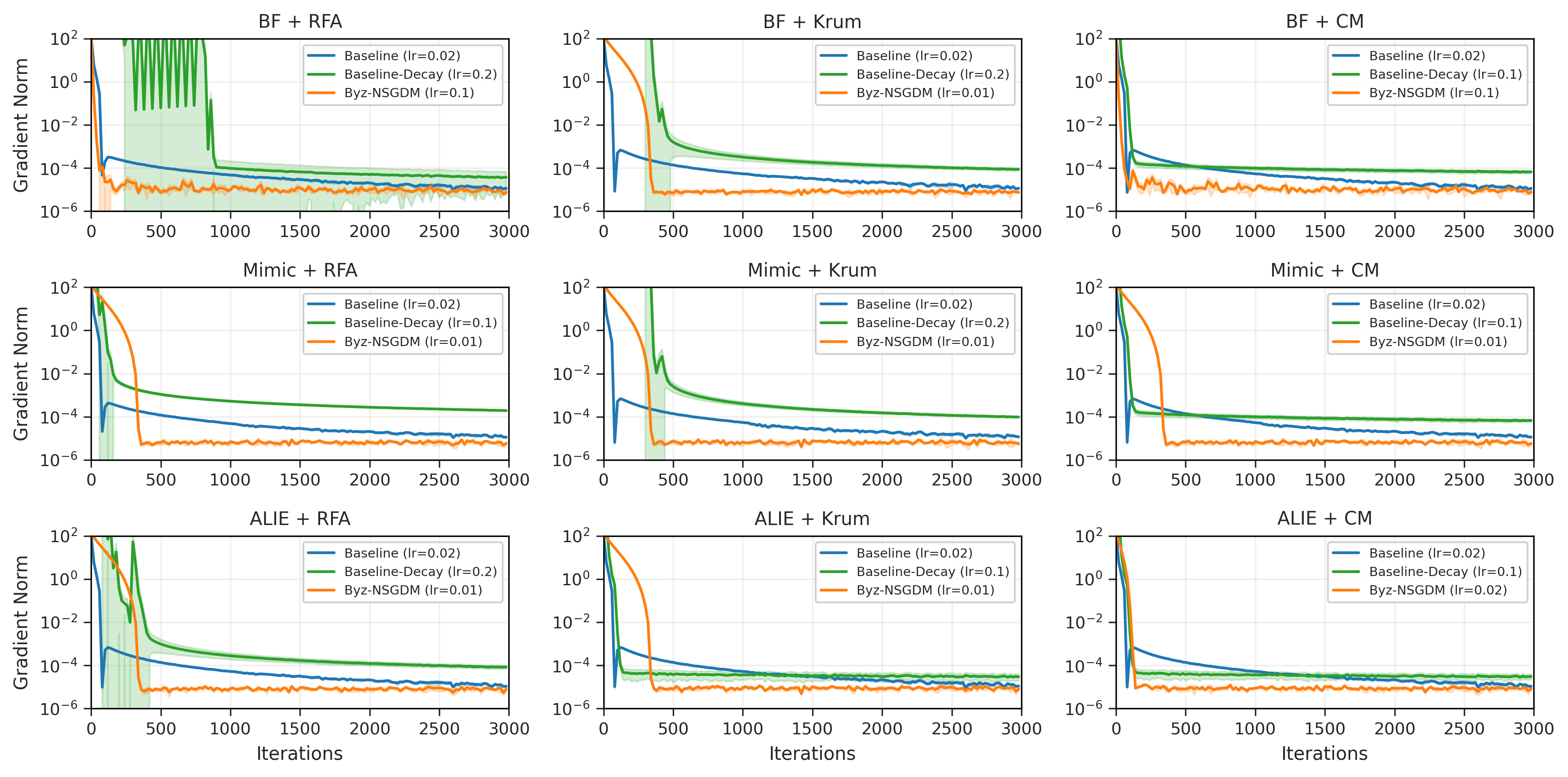}
    \caption{Gradient norm evolution (log scale) for synthetic $(L_0,L_1)$-smooth optimization under Byzantine attacks. Each row shows a different attack (BF, Mimic, ALIE), and each column a different aggregator (RFA, Krum, CM). Lines show mean gradient norm over 3 seeds with shaded regions indicating $\pm 1$ standard deviation. Legend includes tuned learning rates.}
    \label{fig:synthetic_curves}
\end{figure}

\begin{figure}[H]
    \centering
    \includegraphics[width=\textwidth]{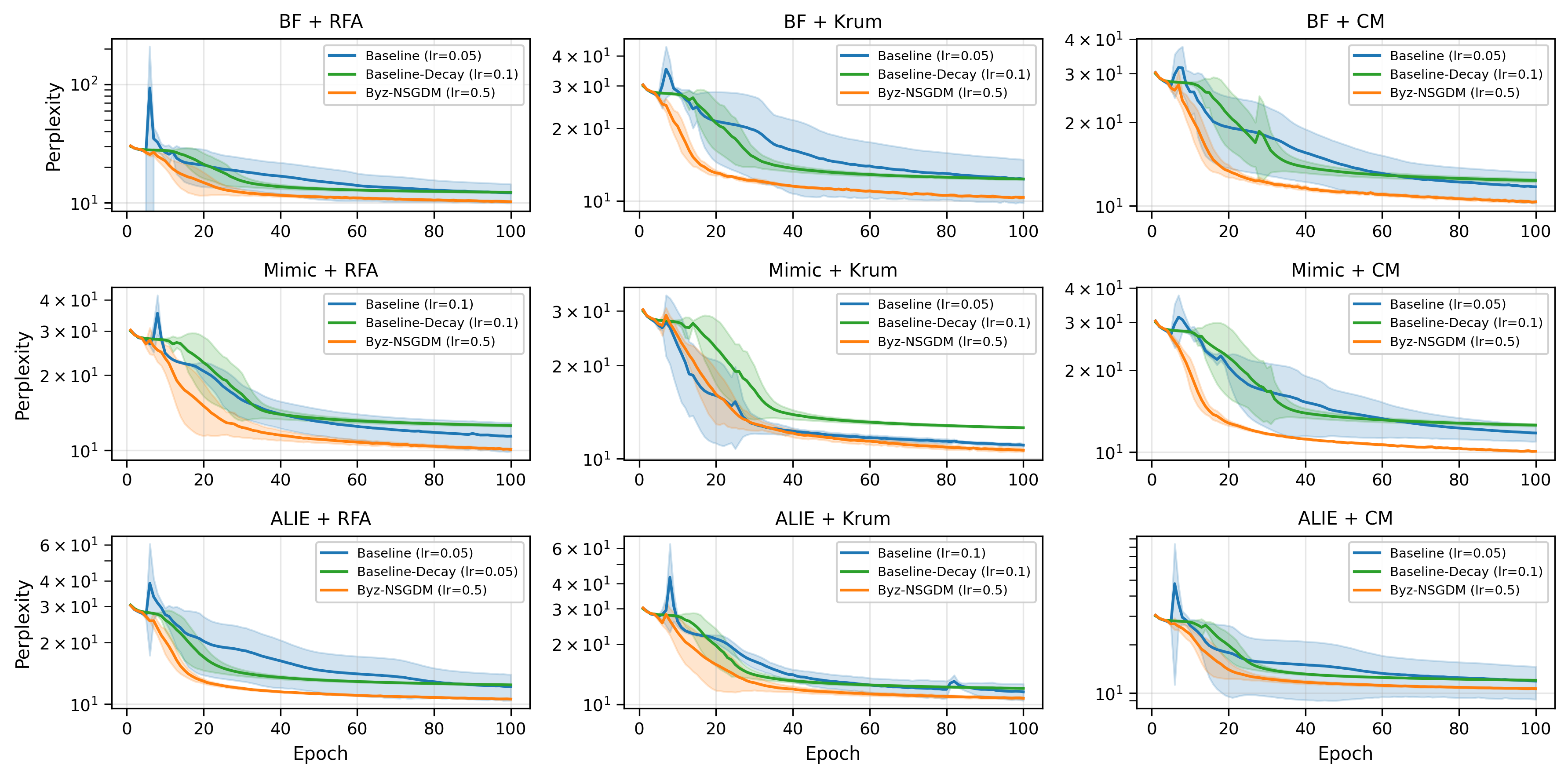}
    \caption{Validation perplexity (log scale) on Shakespeare character-level language modeling under Byzantine attacks. Each row shows a different attack (BF, Mimic, ALIE), and each column a different aggregator (RFA, Krum, CM).}
    \label{fig:lm_curves}
\end{figure}

% Bibliography is handled by the main CPAL document.

\end{document}